\documentclass{article}

\makeatletter
\def\input@path{{assets/}}
\makeatother
\usepackage[preprint]{arxiv}

\usepackage[utf8]{inputenc} %
\usepackage[T1]{fontenc}    %
\usepackage{hyperref}       %
\usepackage{url}            %
\usepackage{booktabs}       %
\usepackage{amsfonts,amsmath,amssymb} %
\usepackage{amsthm}                 %
\usepackage{nicefrac}       %
\usepackage{microtype}      %
\usepackage{xcolor}         %
\usepackage{xspace}         %
\usepackage{graphicx}
\usepackage{multirow}
\usepackage[capitalize]{cleveref}
\usepackage{algorithm}
\usepackage{algpseudocode}
\usepackage{listings}
\usepackage{wrapfig}
\usepackage{tcolorbox}
\tcbuselibrary{skins,breakable}
\usepackage[table]{xcolor}
\usepackage{tikz}
\usepackage{enumitem}
\setlist[itemize]{leftmargin=1.5em}

\makeatletter
\DeclareRobustCommand\onedot{\futurelet\@let@token\@onedot}
\def\@onedot{\ifx\@let@token.\else.\null\fi\xspace}

\def\eg{\emph{e.g}\onedot}

\makeatother

\definecolor{blueMy}{HTML}{1f77b4}
\definecolor{orangeMy}{HTML}{ff7f0e}
\definecolor{greenMy}{HTML}{2ca02c}
\definecolor{redMy}{HTML}{d62728}
\definecolor{purpleMy}{HTML}{9467bd}
\definecolor{brownMy}{HTML}{8c564b}
\definecolor{usccardinal}{rgb}{0.6, 0.0, 0.0}

\newtheorem{theorem}{Theorem}
\newtheorem{lemma}{Lemma}

\makeatletter
\NewDocumentEnvironment{reptheorem}{o m}{%
  \begingroup
  \renewcommand{\thetheorem}{\ref{#2}}%
  \IfNoValueTF{#1}
    {\begin{theorem}}
    {\begin{theorem}[#1]}
  \label{#2-repeated}
}{%
  \end{theorem}
  \endgroup
}
\makeatother

\newcommand{\ind}[1]{ i }

\newcommand{\action}{ \boldsymbol{a} }
\newcommand{\actionToken}{ {a} }

\newcommand{\traj}{ \boldsymbol{\tau} }
\newcommand{\trajToken}{ {\tau} }

\newcommand{\gammaTurn}{ \boldsymbol{\gamma} }
\newcommand{\gammaToken}{ {\gamma} }

\newcommand{\reward}{ \boldsymbol{r} }

\newcommand{\adv}{ \boldsymbol{A}^{\pi_\theta} }
\newcommand{\advToken}{ A^{\pi_\theta}  }

\newcommand{\advFull}{ \adv( \traj_t, \action_t ) }
\newcommand{\advTokenFull}{ \advToken( \trajToken_t^i, \actionToken_t^i )  }

\renewcommand{\hat}[1]{\widehat{#1}}

\renewcommand{\triangleq}{:=}

\newcommand{\commentTMP}[1]{}

\newcommand{\todofuture}[1]{}

\newcommand{\todom}[1]{}
\newcommand{\tocheck}[1]{}

\newcommand{\todo}[1]{}

\newcommand{\instruct}[1]{}

\definecolor{rowrl}{HTML}{91bfdb}

\definecolor{ayubg}{HTML}{FAFAFA}       
\definecolor{ayutext}{HTML}{5C6773}     
\definecolor{ayucomment}{HTML}{ABB0B6}  
\definecolor{ayukeyword}{HTML}{FA8D3E}  
\definecolor{ayustring}{HTML}{86B300}   
\definecolor{ayuborder}{HTML}{E6E6E6}

\lstdefinestyle{ayulight}{
    backgroundcolor=\color{ayubg},   
    basicstyle=\ttfamily\small\color{ayutext}, %
    columns=fullflexible,            %
    keepspaces=true,                 %
    commentstyle=\color{ayucomment},
    keywordstyle=\color{ayukeyword}\bfseries,
    stringstyle=\color{ayustring},
    breakatwhitespace=false,         
    breaklines=true,                 
    captionpos=b,                    
    numbers=none,                    
    showspaces=false,                
    showstringspaces=false,
    showtabs=false,                  
    tabsize=2,
    frame=single, 
    rulecolor=\color{ayuborder} 
}

\definecolor{promptbrown}{HTML}{8B4513}
\definecolor{promptpurple}{HTML}{7030A0}
\definecolor{promptblue}{HTML}{4A75A4}
\definecolor{promptred}{HTML}{C0504D}
\newtcolorbox{academicprompt}[1]{
    enhanced, breakable,
    title={#1},
    colback=white,
    colframe=black,
    boxrule=1.5pt,
    arc=5pt,
    outer arc=5pt,
    left=6pt, right=6pt, top=6pt, bottom=6pt,
    fontupper=\rmfamily\scriptsize,
    coltitle=black,
    fonttitle=\rmfamily\scriptsize\bfseries,
    attach boxed title to top left={yshift=-2mm, xshift=6pt},
    boxed title style={colback=white, colframe=black, boxrule=1pt, arc=3pt}
}
\newcommand{\promptdashline}{
    \vspace{1ex}
    \noindent\tikz\draw[dashed, black!60, line width=0.5pt] (0,0) -- (\linewidth,0);
    \vspace{1ex}
}
\newcommand{\promptsolidline}{
    \vspace{1ex}
    \noindent\tikz\draw[solid, black!80, line width=0.5pt] (0,0) -- (\linewidth,0);
    \vspace{1ex}
}

\title{Hybrid Advantage Estimation with Unified Critic for VLM Agentic Reinforcement Learning}
\newcommand{\ours}{HyGAE}

\author{Wenxuan Zhang \: Yuhui Wang\thanks{Corresponding authors.} \: Donggang Jia \: Xiaoqian Shen \: Jian Ding \: 
\\[0.2em]
\textbf{Ivan Viola \:  Jürgen Schmidhuber \: Mohamed Elhoseiny\textsuperscript{*}}
\\[0.5em]
KAUST \\[0.5em]
{\footnotesize
\texttt{\{wenxuan.zhang,yuhui.wang,donggang.jia,xiaoqian.shen,jian.ding\}@kaust.edu.sa}}\\
{\footnotesize
\texttt{\{ivan.viola,juergen.schmidhuber,mohamed.elhoseiny\}@kaust.edu.sa}
}
}

\begin{document}

\maketitle

\begin{abstract}
    Large Vision-Language Models (VLMs) now act as agents in interactive environments, where success requires coherent reasoning and decision-making across turns. Although end-to-end training in agentic environments can improve such multi-turn decision-making abilities, current methods mainly rely on either token-wise optimization over concatenated token trajectories or turn-wise optimization with uniform within-turn credit. 
    In this work, we establish theoretical formulations for the two levels of optimization and derive a hybrid advantage that serves both objectives. Furthermore, with an appropriate choice of discount factor and learning target, we prove that a unified critic model can estimate values for both turn-wise and token-wise. As such, we propose \ours{}, an actor-critic framework that jointly optimizes token- and turn-level objectives with the hybrid advantage and unified critic. 
    We conduct extensive evaluations of \ours{} across five multi-turn decision-making environments, where it achieves an average success rate of 91\% and a significant improvement of 10\% over other methods. Furthermore, we provide an in-depth analysis showing that the exact analytic form of the hybrid advantage and return is crucial for optimization.  Project Page: \url{https://wx-zhang.github.io/hygae-web/}. 
\end{abstract}

\section{Introduction}

With the rapid evolution of Large Vision-Language Models (VLMs)~\citep{visualgpt,minigpt4,minigptv2,llava}, they are no longer limited to single-turn visual question answering. VLMs are now expected to serve as agents that perceive, reason, and act in the real world~\citep{gpt5,gemini25pro,claude45sonnet}. In this setting, effective agents need to understand complex environments, make sequential decisions, receive feedback, and adapt their future actions over multiple turns. 

A primary bottleneck of current VLMs lies in their nearsightedness: they often reason based solely on the most recent observation, neglecting the critical context of past actions and feedback. As illustrated in \cref{fig:qualitative}, when faced with identical observations and feedback indicating that the previous action failed, the model redundantly repeats its previous failed output rather than synthesizing its experience. In this way, multi-turn decision-making is reduced to a series of disjointed single-turn generations, where task success becomes highly opportunistic rather than the result of systematic strategic planning.

This limitation mainly arises from the mismatch between conventional single-turn training paradigms and the requirements of multi-turn interaction. In multi-turn interaction, models can receive intermediate information that guides them to refine or maintain their internal beliefs. Several recent works have noticed this issue and attempted to design algorithms that directly support multi-turn agentic reinforcement learning~\citep{vagen,rl4vlm,chu2025sft,wang2025ragen}. A central discussion in this context concerns two levels of reinforcement learning modeling: (1) token-wise RL, which treats each token as a single action, and (2) turn-wise RL, which treats the turn as a whole as one action. Turn-wise optimization better leverages intermediate feedback for intra-turn differentiation, while token-wise training provides fine-grained inner-turn guidance for language generation~\citep{chen2025reinforcement}. 
More recently, several works have adapted classical hierarchical reinforcement learning ideas~\citep{schmidhuber1990hrl0,schmidhuber1991hrl1,schmidhuber1992hrl2} to combine token-level and turn-level optimization in multi-turn agentic RL~\citep{zhou2024archer,vagen}.
To better understand these RL modeling choices, we need to answer an essential question: what is the inherent difference between the two levels of RL modeling?

\begin{figure}[t]
    \centering
    \includegraphics[width=\linewidth]{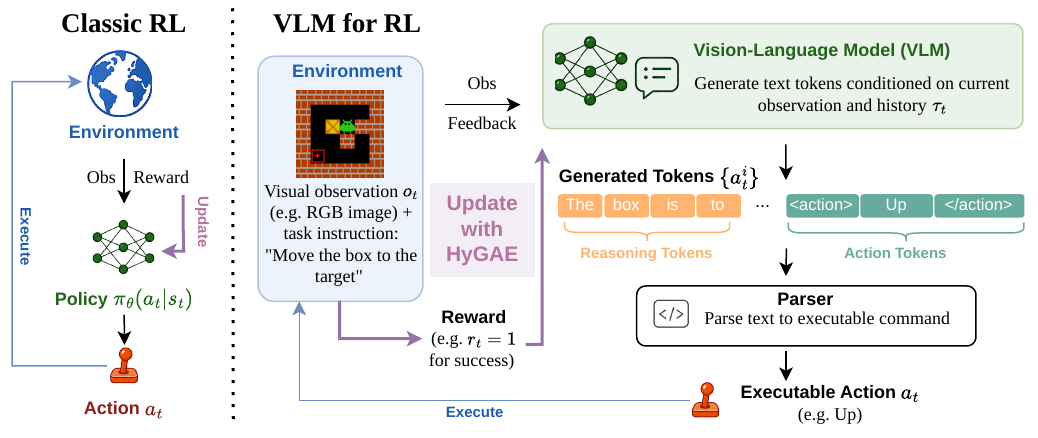}
    \caption{Multi-turn decision-making tasks for VLMs. We update the VLM policy using hybrid advantages from the token and turn levels, while learning the value with a unified value model.}
    \label{fig:overview}
\end{figure}
To answer this question, we first provide a POMDP-based theoretical formulation of token-wise and turn-wise RL, showing that the two formulations differ only in advantage estimation and that their surrogates therefore have the same gradient. Based on this observation, we propose hybrid advantage estimation, which optimizes the two levels of RL simultaneously through a simple linear combination and improves both inner- and intra-turn credit assignment. Furthermore, we show that, with an appropriate choice of discount factor and learning target, we can learn a unified critic model for both turn-wise and token-wise value estimates.
We instantiate the resulting hybrid advantage and value estimation in an actor-critic framework, denoted as \ours{}, which benefits from both language-generation training and turn-wise intermediate feedback during training, as shown in \cref{fig:overview}. 
We further prove that hybrid advantage estimation and unified value training provide a better balance between bias and variance.

We evaluate \ours{} on five multi-turn decision-making benchmarks and compare our algorithm with various baselines. The results show that \ours{} achieves SOTA performance on all tasks, with an average success rate of 0.91. We further provide ablations and qualitative analyses showing that the precise analytic forms of the hybrid advantage and mixed return are crucial for stable optimization. Our contributions are as follows: 
\begin{itemize}
    \item We provide a formal formulation of token-wise and turn-wise RL in multi-turn decision-making tasks and reveal the intrinsic similarities between them.
    \item We design hybrid advantage estimation with unified value model training to optimize the policy for better inner- and intra-turn credit assignment. This formulation is computationally efficient and provides a better bias-variance trade-off.
    \item We achieve SOTA performance on five multi-turn decision-making tasks, reaching an average success rate of 0.91.
\end{itemize}

\section{Related Work}

\paragraph{Multi-turn Decision-Making Tasks for VLMs.}
 Early approaches for solving vision-related multi-turn decision-making tasks mainly fall into two categories: (1) leveraging pretrained visual representations, \eg, CLIP~\citep{clip}, for decision-making, often trained with external reward models \citep{genrl,minedojo}, and (2) directly employing pretrained VLMs as agents for sequential reasoning and action generation \citep{zhou2024minedreamer}.

Recent works increasingly adopt reinforcement learning to train VLMs for these tasks. Some works propose to obtain high-quality training data from external sources \citep{jin2023alphablock,kalie,yang2024octopus,faldoromni}, self-refining data through interaction or feedback \citep{zhang2025mem2ego,zhou2025proposer,sarch2024vlm}, improving reward modeling \citep{digiq,yanglearning,liu2025vlp}, and integrating tool usage during rollout \citep{liu2025visual,lu2025scaling}. In addition, \citet{cosmosr17b} proposes a comprehensive pipeline spanning pre-training to post-training. Despite these advances, most methods still rely on standard RL algorithms such as PPO~\citep{schulman2017proximal} or GRPO~\citep{guo2025deepseek}.

\paragraph{Multi-turn Reinforcement Learning for VLMs.}
An important line of work focuses on the design of training algorithms for foundation models in multi-turn decision-making tasks. Some studies analyze the underlying mechanisms of foundation models in such scenarios \citep{klissarovmodeling,chu2025sft}, while others propose improved optimization strategies, such as advantage-weighted RL with stochasticity-aware estimators \citep{digirl}.  
A central topic in this line of research is how to model trajectories at the turn level or token level \citep{chen2025reinforcement}. On one hand, some approaches introduce complex and computationally expensive hierarchical models to explicitly capture both levels, such as \citet{zhou2024archer}. On the other hand, many methods operate at a single level and rely on heuristic designs for credit assignment. \citet{rl4vlm} introduces a coefficient to balance the contribution of reasoning (thinking) tokens and action tokens, and follow-up works further refine control over reasoning processes \citep{wei2025gtr,wei2025gtr-turbo}. Similarly, \citet{vagen} applies different discount factors to token-level and turn-level rewards, and \citet{fenggroup} analyzes credit assignment in GRPO at both the turn and sequence levels.
In contrast to these approaches, through theoretical analysis, we find that the two levels of RL are intrinsically related. We therefore propose hybrid advantage estimation, which combines token-wise and turn-wise advantages with a unified value function and achieves a better balance between bias and variance.

\section{Preliminaries}

\paragraph{POMDP}
Following~\citep{vagen}, we define the multi-turn interaction of VLMs as a Partially Observable Markov Decision Process (POMDP)  $(\mathcal{S}, \mathcal{A}, \mathcal{O}, \mathcal{T}, \mathcal{R}, \Omega, \gammaTurn)$. Here, $\mathcal{S}$, $\mathcal{A}$, $\mathcal{O}$, and $\mathcal{R}$ denote the state, action, observation, and reward spaces, respectively; $\mathcal{T}$ represents the state transition dynamics; $\Omega$ represents the observation function; and $\gammaTurn \in [0,1)$ is the discount factor. 
At each \emph{decision step} $t$, the agent receives an observation
$\boldsymbol{o}_t \in \mathcal{O} \sim \Omega(\,\cdot \mid \boldsymbol{s}_t)$
conditioned on the system state $\boldsymbol{s}_t \in \mathcal{S}$ and generates an action
$\boldsymbol{a}_t \in \mathcal{A}$ conditioned on the history $\tau_{t}^{i} \triangleq \left(\boldsymbol{o}_{0}, \boldsymbol{a}_{0}, \boldsymbol{r}_{0}, \boldsymbol{o}_{1}, \boldsymbol{a}_{1}, \boldsymbol{r}_{1}, \dotsc, \boldsymbol{o}_{t}, \left(a_{t}^{0}, a_{t}^{1}, \dotsc, a_{t}^{i}\right)\right)$.
The environment then transitions from state $\boldsymbol{s}_t$ to a new state
$\boldsymbol{s}_{t+1} \sim \mathcal{T}(\cdot \mid \boldsymbol{s}_t, \boldsymbol{a}_t)$,
produces a reward
$\boldsymbol{r}_t \sim \mathcal{R}(\cdot \mid \boldsymbol{s}_t, \boldsymbol{a}_t)$,
and sends the agent a new observation $\boldsymbol{o}_{t+1}$.

\paragraph{RL with Actor-Critic Framework}
The objective of reinforcement learning is to find an optimal policy $\pi_\theta$ parameterized by $\theta$ that maximizes the expected return, defined as the sum of the discounted future rewards:
\begin{equation}\label{eq_objective}
J(\theta) 
=  \mathbb{E}_{\pi_\theta}
\left[   
    \sum_{t=0}^{ T-1 } {\gamma}^t \boldsymbol{r}_t
\right]
\end{equation}
To optimize this objective, it is common to employ the surrogate function:
\begin{equation}\label{eq_policy_gradient}
L_\text{actor}( \theta ) =\mathbb{E}\left[\sum _{t=0}^{T-1}\log \pi _{\theta }(\boldsymbol{a}_{t} |\boldsymbol{\tau }_{t})\boldsymbol{A}^{\pi _{\theta }} (\boldsymbol{\tau }_{t} ,\boldsymbol{a}_{t} )\right]
\end{equation}
where
$
\boldsymbol{A}^{\pi _{\theta }} (\boldsymbol{\tau }_{t} ,\boldsymbol{a}_{t} )=\boldsymbol{Q}^{\pi _{\theta }}(\boldsymbol{\tau }_{t} ,\boldsymbol{a}_{t}) -\boldsymbol{V}^{\pi _{\theta }}(\boldsymbol{\tau }_{t})
$ is the advantage value function defined as
\begin{equation*}
\boldsymbol{Q}^{\pi _{\theta }}(\boldsymbol{\tau }_{t} ,\boldsymbol{a}_{t}) 
\triangleq
\mathbb{E}\left[\sum _{t'=t}^{T-t }\boldsymbol{\gamma }^{t'-t}\boldsymbol{r}_{t'} \mid \boldsymbol{\tau }_{t} ,\boldsymbol{a}_{t}\right]
,
V^{\pi _{\theta }}(\boldsymbol{\tau }_{t}) 
\triangleq
\mathbb{E}_{a\sim\pi_\theta(\cdot|\boldsymbol{\tau }_{t})}\left[\boldsymbol{Q}^{\pi _{\theta }}\left(\boldsymbol{\tau }_{t} ,a\right)\right]
\end{equation*}
In practice, the advantage value $\advFull$ is usually estimated by the Generalized Advantage Estimation (GAE)~\cite{gae}:
\begin{equation}\label{eq:gae}
\widehat{\boldsymbol{A}}_t=\; \sum_{t'=0}^{T-t}(\gamma\lambda)^{t'}\,\delta_{t+t'},
\text{ where }
\delta_t \;:=\; \boldsymbol{r}_t + \gamma \widehat{\boldsymbol{V}}_\psi(\boldsymbol{\tau_{t+1}}) - \widehat{\boldsymbol{V}}_\psi(\boldsymbol{\tau_t}),
\end{equation}
where $\delta_t$ is the temporal difference error, and $\lambda \in [0,1]$ is a secondary discount factor that controls the bias–variance trade-off. 
$\widehat{\boldsymbol{V}}_\psi$ is a parameterized value function that estimates the state value. 

We employ the actor-critic framework for the value function estimation and policy update.  

$\widehat{\boldsymbol{V}}_\psi$ is also called the critic, which is updated together with the actor using the TD error as the target.
\begin{equation}\label{eq:critic_loss}
    L_{\text{critic}}(\psi)
    =
    \mathbb{E}_{t}\bigl[(\widehat{\boldsymbol V}_\psi(\boldsymbol{s}_t)-\widehat{\boldsymbol G}_t)^2\bigr],
    \text{ where }
    \widehat{\boldsymbol{G}}_t = \widehat{\boldsymbol{A}}_t + \widehat{\boldsymbol{V}}_\psi(\boldsymbol{\tau_t}).
\end{equation}

\section{Method}\label{sec_method}

As discussed above, few prior works have formally formulated the RL process in VLM agentic training. We first derive turn-wise and token-wise objectives under a unified view, revealing their shared surrogate structure and motivating our hybrid actor-critic framework with \ours{} advantage estimation and unified value model training. Finally, we provide a theoretical analysis explaining why the hybrid approach is both effective and efficient.

\paragraph{Notation:}
Throughout this paper, we use boldface notation to represent turn-wise variables, such as $\action_t$, and standard unbolded font with a superscript local token index $i$ to represent token-wise variables, such as $\actionToken_t^i$.
Full definitions, theorems, and proofs are provided in Appendix A. 

\subsection{Turn- and Token-wise Formulations}
\label{sec:turn_wise_and_token_wise_surrogate}
In multi-turn decision-making tasks,
at each turn $t \in [0,T-1]$, the VLM policy generates a token sequence
$\action_t \triangleq (\actionToken_t^1, \actionToken_t^2, \ldots, \actionToken_t^{I_t})$,
where $\actionToken_t^i$ is the $i$-th token generated at turn $t$, and $I_t$ is the length of the token sequence for turn $t$. The generated tokens encode intermediate reasoning and executable actions. The environment then executes the action, providing a reward $\boldsymbol{r_t}$ and a feedback token sequence $\boldsymbol{o_{t+1}}$ for the next turn.

This generation–interaction loop can be viewed at either the turn or token scale, leading to distinct POMDP formulations and corresponding RL training strategies. Below, we detail these two perspectives.

\paragraph{Token-wise POMDP.} 
A popular perspective treats the generation of each token as an individual action~\cite{vagen, rlhf,lu2025scaling}. 
At each decision step ($t$-th turn and $i$-th token generation),
the VLM policy $\pi_\theta$ takes the context $\trajToken_t^i$ as input and samples a token $a_t^i \sim \pi_\theta( \,\cdot\, | \trajToken_t^i ).$ 
The token-wise reward $r_t^i$ assigned to each token is defined as
$$
{r}_{t}^{i} =\begin{cases}
\boldsymbol{r}_{t} , & i=I_{t}\\
0, & \text{otherwise}
\end{cases}. 
$$
This indicates that if the token is at the end of the turn, its reward is the environmental reward $\reward_t$; otherwise, it is $0$ for intermediate tokens.

Based on this formulation, \cref{eq_objective} can be rewritten by re-indexing across both turn and token indices:
$
J(\theta )=\mathbb{E}_{\pi _{\theta }}\left[\sum _{t=0}^{T-1}\sum _{i=0}^{I_{t} -1}{\gamma }^{i+I_{<t} }{r}_{t}^{i}\right],
$ where $I_{<t} = \sum_{k=0}^{t-1} I_k$ is the total number of tokens generated before the $t$-th turn, and $\gammaToken$ is the discount factor. From this objective, we can derive the token-wise surrogate function:
\begin{equation}\label{eq_surrogate_token}
L^{\mathrm{token}} (\theta )=\mathbb{E}\left[\sum _{t=0}^{T-1}\sum _{i=0}^{I_{t} -1}\log \pi _{\theta }\left( a_{t}^{i} |\tau _{t}^{i}\right) A^{\pi _{\theta }}\left( \tau _{t}^{i} ,a_{t}^{i}\right)\right]
\end{equation}
Here ${A}^{\pi _{\theta }}(\trajToken_t^i, \actionToken_t^i)$ represents the token-wise advantage function, which estimates the advantage value for each generated token.

\paragraph{Turn-wise POMDP.}
This perspective treats the generation of the entire token sequence $\action_t$ within a single turn as one unified macro-action~\cite{rl4vlm}. 
We define the turn-wise POMDP as follows:
at decision step $t$ (the $t$-th turn), 
the VLM policy takes the context $\boldsymbol{\tau}_{t}$ as input and generates $\action_t \sim \pi_\theta(\,\cdot\, | \traj_t )$. 
The context $\boldsymbol{\tau}_{t}$ is constructed from the observations, actions, and rewards returned by the environment:
\begin{equation}\label{eq_traj}
\boldsymbol{\tau}_{t} \triangleq (\boldsymbol{o}_{0}, \boldsymbol{a}_{0}, \boldsymbol{r}_{0}, \boldsymbol{o}_{1}, \boldsymbol{a}_{1}, \boldsymbol{r}_{1}, \dotsc, \boldsymbol{o}_{t})
\end{equation}
The reward in the turn-wise formulation is simply the reward $\boldsymbol{r_t}$ returned by the environment.
The objective follows \cref{eq_objective}, and we can use the surrogate function from \cref{eq_policy_gradient} to optimize it.

\subsection{Hybrid Advantage Function and Unified Value Function}

Note that the turn action, $\boldsymbol{a}_t \triangleq (\actionToken_t^i)_{i=0}^{I_t}$,  is generated autoregressively by the model. The log-probability of generating the turn-wise action, $\log \pi_{\theta}(\boldsymbol{a}_{t} | \boldsymbol{\tau}_{t})$, can be decomposed into the sum of the log-probabilities of generating each token:
\begin{equation}\label{eq_log_probability_decompose}
\log \pi_{\theta}(\boldsymbol{a}_{t} | \boldsymbol{\tau}_{t}) = \sum_{i=0}^{I_{t}-1} \log \pi_{\theta}\left(a_{t}^{i} | \tau_{t}^{i}\right),
\end{equation}
Substituting \Cref{eq_log_probability_decompose} into the surrogate function in \cref{eq_policy_gradient}, we obtain the following form of the turn-wise surrogate:
\begin{equation}\label{eq_surrogate_turn}
L^{\mathrm{turn}} (\theta )=\mathbb{E}\left[\sum _{t=0}^{T-1}\sum _{i=0}^{I_{t} -1}\log \pi _{\theta }\left( a_{t}^{i} |\tau _{t}^{i}\right)\boldsymbol{A}^{\pi _{\theta }} (\boldsymbol{\tau }_{t} ,\boldsymbol{a}_{t} )\right]
\end{equation}
where $\boldsymbol{A}^{\pi_{\theta}}(\traj_t, \action_t)$ represents the turn-wise advantage function, which estimates the advantage value for the turn-wise action.

The turn-wise and token-wise surrogates in \cref{eq_surrogate_turn} and \cref{eq_surrogate_token} have the same form except for the advantage value: the turn-wise surrogate employs the turn-wise advantage $\advFull$, whereas the token-wise surrogate uses the token-wise advantage $\advTokenFull$.

\begin{figure}[t]
    \centering
    \includegraphics[width=\textwidth]{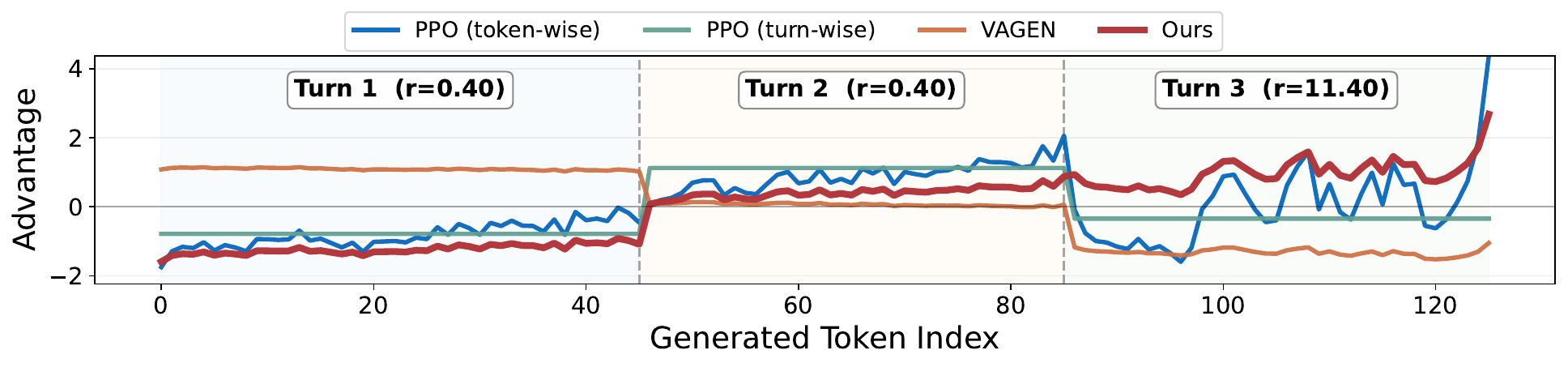}
    \caption{Estimated advantage scales for different advantage estimators over all tokens in the trajectory.}
    \label{fig:advantages}
\end{figure}
The two estimators focus on different properties during training. 
As shown in \cref{fig:advantages}, the turn-wise advantage estimation remains constant across the turn. This preserves the turn-level reward signal and avoids excessive discounting through intermediate reasoning tokens. Conversely, the token-wise advantage estimation provides fine-grained advantage values for each individual token, preserving within-turn differentiation.

Existing methods typically employ either estimator independently or use one signal to guide the other when both advantages are considered. The equivalence above provides a theoretical basis for optimizing with both advantage signals without introducing hierarchical modeling; instead, the two signals can be directly hybridized. This leads to a \emph{hybrid surrogate function} using a hybrid advantage function:
\begin{fontsize}{9.5}{1}
\begin{equation}\label{eq_surrogate_hybrid}
L^{\mathrm{hybrid}} (\theta )=\mathbb{E}\left[\sum _{t=0}^{T-1}\sum _{i=0}^{I_{t} -1}\log \pi _{\theta }\left( a_{t}^{i} |\tau _{t}^{i}\right)\left( \alpha \boldsymbol{A}^{\pi _{\theta }} (\boldsymbol{\tau }_{t} ,\boldsymbol{a}_{t} )+( 1-\alpha ) A^{\pi _{\theta }}\left( \tau _{t}^{i} ,a_{t}^{i}\right)\right)\right]
\end{equation}
\end{fontsize}\ignorespacesafterend
    
Formally, the following theorem demonstrates that the hybrid surrogate yields the same policy gradient as the original objective $J(\theta)$.
\begin{theorem}\label{theorem_gradient}
    Let the discount factors of the turn-wise and token-wise formulations satisfy $\gammaTurn_t = \gammaToken^{I_t}$, where $I_t$ is the length of turn $t$. Then, for any $\alpha \in [0,1]$ in \cref{eq_surrogate_hybrid}, we have
    \[
    \nabla J(\theta)
    =
    \nabla L^{\mathrm{hybrid}}(\theta)
    =
    \nabla L^{\mathrm{turn}}(\theta)
    =
    \nabla L^{\mathrm{token}}(\theta).
    \]
\end{theorem}
As we will show in \cref{sec_theorem}, the hybrid advantage can also balance the bias--variance tradeoff.
Therefore, utilizing the hybrid advantage is theoretically sound and effective compared to relying on a single surrogate for advantage estimation.

Naively, the turn-wise and token-wise advantage computations involve different value models, $\boldsymbol{\widehat{V}_\psi}$ and $\widehat{V}_\psi$. One would need to train two separate value models within the actor-critic framework, which introduces significant computational overhead and memory usage. However, we demonstrate that by appropriately setting the discount factors $\gammaTurn$ and $\gammaToken$, a single, unified value function can be employed for the estimation of both turn-wise and token-wise advantages.

\begin{theorem}\label{thm:unified_value_function}
Let the discount factors of the turn-wise and token-wise formulations satisfy $\gammaTurn_t = \gammaToken^{I_t}$, where $I_t$ is the length of turn $t$. Then the corresponding value functions are consistent at the last token of the trajectory, that is,
\begin{equation}\label{eq:value_function_relation}
    \boldsymbol{V}^\pi(\boldsymbol{\tau}_t) = V^\pi(\tau_t^{I_t}) \quad \text{for any turn } t .
\end{equation}
\end{theorem}
\cref{thm:unified_value_function} shows that, given the appropriate discount factors, the turn-wise value $\boldsymbol{V}^\pi(\boldsymbol{\tau}_t)$ is exactly equal to the token-wise value of the final token $V^\pi(\tau_t^{I_t})$. This equivalence implies that we can use exactly one shared value model, $V_\psi$, which provides turn-wise values at the end of the turn and token-wise values at intermediate token positions.
\begin{figure}[t]
    \centering
    \includegraphics[width=.95\linewidth]{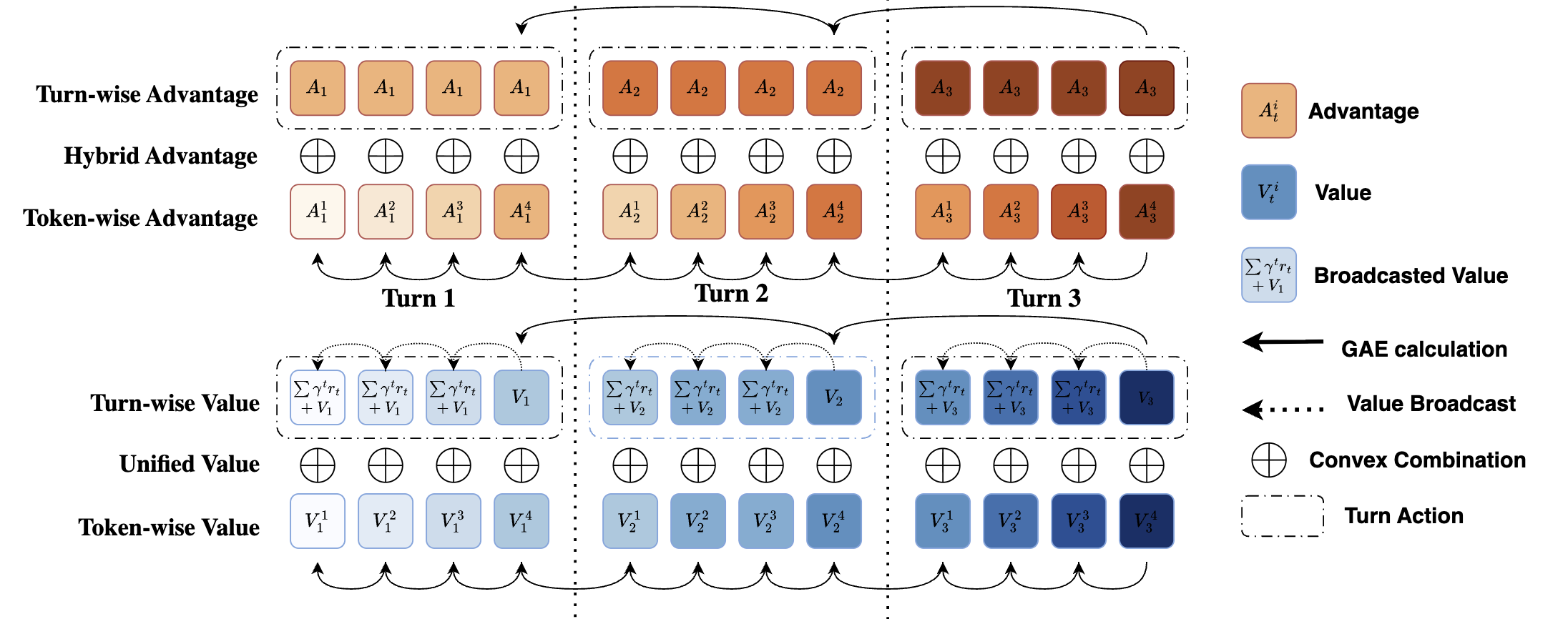}
    \caption{Actor-critic framework with \ours{}: we combine the token-wise advantage and turn-wise advantage to update the policy model, while broadcasting the turn value with discounted rewards to learn a unified value model.}
    \label{fig:algorithm}
\end{figure}
\subsection{\ours{} Estimation}\label{sec:estimation}
\begin{algorithm}[t]
    \caption{\ours{}}
    \label{alg:mixed_advantage_training}
    \begin{algorithmic}[1]
        \footnotesize 
        \Statex \textbf{Input:} Policy $\pi_\theta$; Value Model $\widehat{V}_\psi$.
    \For{each iteration $k = 0$ to $K$ }
        \State  Interact with the environment and collect trajectory $\traj$ (see \cref{eq_traj}), which contains $T$ turns.
        \For{each turn $t = T-1$ to $0$  }  
            \State Obtain turn-wise advantage $\widehat{\boldsymbol{A}}_t$ and return $\widehat{\boldsymbol{G}}_t$ via turn-wise GAE in \cref{eq_advantage_turn} and \cref{eq_return_turn}.
            \State Obtain token-wise advantage $\widehat{A}_t^i$ and return $\widehat{G}_t^i$ via token-wise GAE in \cref{eq_advantage_token} and \cref{eq_return_token}.
            \For{each token $i = I_t-1$ to $0$}\Comment{\ours{} advantage estimation}
            \State Calculate the hybrid advantage $\widehat{\mathsf{A}}_t^i$ via \cref{eq_advantage_hybrid}.
            \State Construct the training target for the turn-wise value model by broadcasting the turn-wise return with discounted rewards to each token position via \cref{eq_return_turn}.
            \State Calculate the hybrid return ${\widehat{\mathsf{G}}_t^i}$ via \cref{eq:mixed_return}.
            \EndFor
        \EndFor
        \State Update the policy $\pi_\theta$ and value model $\widehat{V}_\psi$ via the policy loss in \cref{eq_surrogate_practice} and the TD loss with the target in \cref{eq:mixed_return}. 
    \EndFor
        \State Return policy $\pi_\theta$ and value model $\widehat{V}_\psi$.
    \end{algorithmic}
\end{algorithm}

In practice, following \cref{eq:gae}, we can estimate the hybrid advantage value by combining the turn-wise and token-wise advantage estimates. As shown in \cref{fig:algorithm}, we first calculate the turn-wise advantage estimate by 
\begin{equation}\label{eq_advantage_turn}
    \widehat{\boldsymbol{A}}_{t} =\sum _{t'=0}^{T-t}(\boldsymbol{\gamma } \lambda )^{t'}\boldsymbol{\delta }_{t'},
    \quad
    \text{where }
    \boldsymbol{\delta }_{t'} =\boldsymbol{r}_{t'} +\boldsymbol{\gamma }\boldsymbol{\widehat{V}_\psi}(\boldsymbol{\tau }_{t'+1}) -\boldsymbol{\widehat{V}_\psi}(\boldsymbol{\tau }_{t'})
    \end{equation}
We then calculate the token-wise advantage value $\advTokenFull$ by 
\begin{fontsize}{9.3}{1}
\begin{equation}\label{eq_advantage_token}
\hat{{A}}_{t}^{i} =\sum _{i'=0}^{I-i-I_{< t}}( \gamma \lambda )^{i'} \delta ^{( i'+i+I_{< t})},
\text{ where }
\delta ^{( i')} =r^{( i')} +\gamma \hat{V}_\psi\left( \tau ^{( i'+1)}\right) -\hat{V}_\psi\left({\tau }^{( i')}\right)
\end{equation}
\end{fontsize}\ignorespacesafterend
For simplicity, we use the superscript $(i')$ to denote the global token index in the whole trajectory. Here, $\hat{V}_\psi$ is the value model in the actor-critic framework.

Finally, we combine the turn-wise and token-wise advantage values to obtain the hybrid advantage value $\hat{\mathsf{A}}_t^i$ by \cref{eq_advantage_hybrid}.
\begin{equation}\label{eq_advantage_hybrid}
    \hat{\mathsf{A}}_t^i
=
\alpha  \widehat{\boldsymbol{A}}_{t} + (1-\alpha) \widehat{A}_t^i, 
\quad
\alpha \in [0,1]
\end{equation}
where $\alpha$ is a hyperparameter that weights the turn-wise signal against the token-wise signal. As shown in \cref{fig:advantages}, the hybrid advantage estimation provides fine-grained advantage values for each individual token, while being less discounted by the intermediate tokens. 

We employ the clipping technique in Proximal Policy Optimization (PPO) \cite{schulman2017proximal} to improve training stability. The empirical surrogate loss function derived from \cref{eq_surrogate_hybrid} can be written as
\begin{equation}\label{eq_surrogate_practice}
\hat{L}^{\mathrm{hybrid}} (\theta )=\hat{\mathbb{E}}\left[\sum _{t=0}^{T-1}\sum _{i=0}^{I_{t} -1}\min\left( \rho _{t}^{i}( \theta )\hat{\mathsf{A}}_{t}^{i} ,\text{clip} (\rho _{t}^{i}( \theta ) ,1-\epsilon ,1+\epsilon )\hat{\mathsf{A}}_{t}^{i} \ \right)\right]
\end{equation}
where $\rho_t^i(\theta)\triangleq \pi_\theta( \actionToken_t^i | \trajToken_t^i ) / \pi_{\theta'}( \actionToken_t^i | \trajToken_t^i )$ is the likelihood ratio.

A critical practical consideration is the calculation of the target return used to update this shared value model $\widehat{V}_\psi$. While returns from either the turn or token level could be used, they exhibit significantly different bias-variance profiles. 
Computing the token-wise return following \cref{eq:gae} is straightforward,
\begin{equation}\label{eq_return_token}
\widehat{G}_t^i = \widehat{{A}}_t^i + \widehat{V}_\psi(\tau_t^i)
\end{equation}

Conversely, the turn-wise GAE only provides a single return at the end of the turn. To compute the hybrid return in the middle of the turn, we must broadcast the turn-wise return back to each intermediate token position by summing the discounted rewards from the current token index to the end of the turn:
\begin{equation}\label{eq_return_turn}
 \widehat{\boldsymbol{G}}_t^i = \widehat{\boldsymbol{A}}_t + \sum_{k=i+1}^{I_t} \gammaToken^{k-i} r_k + \widehat{V}_\psi(\boldsymbol{\tau}_t)
\end{equation}

Therefore, as illustrated at the bottom of \cref{fig:algorithm}, we calculate a \emph{hybrid return} to combine the turn-wise and token-wise returns:
\begin{equation}\label{eq:mixed_return}
    \widehat{\mathsf{G}}_t^i = \alpha \widehat{\boldsymbol{G}}_t^i + (1-\alpha) \widehat{G}_t^i,
\end{equation}
We demonstrate in \cref{sec:key} that this specific formulation of the hybrid return is crucial for optimizing the performance of the value model.

Regarding environmental consistency, the token-level rewards should formally aggregate perfectly into the turn-level rewards. In practice, token-level rewards often consist of very small KL-divergence penalties. We found it empirically beneficial \emph{not} to aggregate these micro-rewards into the final turn-level reward $R^{(i)}$. We show both of these design choices in \cref{sec:key}.

The full procedure for calculating these mixed signals is summarized in \Cref{alg:mixed_advantage_training}. 
Once the estimated advantage $\mathsf{A}_t^i$ and return $\mathsf{G}_t^i$ are computed, the policy and value networks are updated via the standard PPO loss.

\subsection{Bias and Variance Analysis}\label{sec_theorem}

We compare the bias and variance of the turn-wise and token-wise advantage estimators.
We define the bias of the token-wise advantage estimator $\hat{A}_{t}^{i}$ as
$
\mathrm{bias}\left(\hat{A}_{t}^{i}\right) =\left\Vert \mathbb{E}\left[\hat{A}_{t}^{i} \mid \tau _{t}^{i} ,a_{t}^{i}\right] -A^{\pi } (\tau _{t}^{i} ,a_{t}^{i} )\right\Vert _{\infty },
$
where $\|\cdot\|_{\infty}$ denotes the infinity norm.
Similarly, we can define $\mathrm{bias}(\hat{\boldsymbol{A}}_{t})$ for the turn-wise estimator.

\begin{theorem}[Bias Comparison]\label{theorem_bias_comparision}
Let $e_{\max} \triangleq \| \hat{V} - V^{\pi} \|_{\infty}$, and $I=\min_t I_t$.
Then
\begin{fontsize}{9}{11}
\begin{equation*}
\mathrm{bias}\!\left(\hat{\boldsymbol{A}}_{t}\right)
\le
\underbrace{\left(
1+\frac{\gamma^{I}(1-\lambda)}{1-\gamma^{I}\lambda}
\right)}_{b_{\mathrm{turn}}}
e_{\max},
\ 
\mathrm{bias}\!\left(\hat{A}_{t}^{i}\right)
=
\underbrace{\left(
1+\frac{\gamma(1-\lambda)}{1-\gamma\lambda}
\right)}_{b_{\mathrm{token}}}
e_{\max},
\ 
b_{\mathrm{turn}} \leq b_{\mathrm{token}}.
\end{equation*}
\end{fontsize}
\end{theorem}

\begin{theorem}[Variance Comparison]\label{theorem_variance_comparision}
Let $\sigma_{\max} \triangleq \sup \mathrm{var}(\delta^{(i)})$, and $I=\max_t I_t$.
Then
\begin{fontsize}{9}{11}
\begin{equation*}
\mathrm{var}\!\left(\hat{\boldsymbol{A}}_{t}\right)
\le
\underbrace{\left(
\frac{1-\gamma^{I}}
{(1-\gamma)(1-\gamma^{I}\lambda)}
\right)^{2}}_{v_{\mathrm{turn}}}
\sigma_{\max},
\ 
\mathrm{var}\!\left(\hat{A}_{t}^{i}\right)
\le
\underbrace{\left(
\frac{1}{1-\gamma\lambda}
\right)^{2}}_{v_{\mathrm{token}}}
\sigma_{\max},
\ 
v_{\mathrm{turn}} > v_{\mathrm{token}}.
\end{equation*}
\end{fontsize}
\end{theorem}
The above theorems show that the turn-wise estimator tends to have smaller bias but larger variance compared to the token-wise estimator.
These results reveal a bias--variance tradeoff between the turn-wise and token-wise estimators,
and our hybrid approach balances the two by combining them.
All proofs are provided in Appendix A.

\section{Experiments}

\subsection{Setup}

\paragraph{Environment:}
We evaluate \ours{} across five decision-making benchmarks, following the protocols established in~\cite{vagen,chu2025sft}. These environments require Vision-Language Models (VLMs) to process visual observations and generate natural language actions to interact with the environment.

\begin{itemize}
    \item \textbf{Sokoban}~\cite{sokoban}: Evaluates spatial reasoning and long-term planning. The agent must push boxes onto target positions on a grid using actions: \texttt{up, down, left, right}.
    \item \textbf{FrozenLake(F.Lake)}~\cite{frozenlake}: Tests path planning under constraints. The agent navigates from a start to a goal while avoiding holes, using the same discrete action space as Sokoban.
    \item \textbf{Navigation}~\cite{yang2025embodiedbench,kolve2017ai2}: Assesses visual grounding and instruction following. The agent navigates to a target within 1.5 meters using an 8-action space related to move, turn, and look. We test both direct instructions and complex scenarios involving common-sense reasoning.
    \item \textbf{Primitive Skill}~\cite{tao2024maniskill3,srivastava2025rosetta,nasiriany2022augmenting,hiranaka2023primitive}: Focuses on robotic manipulation. The agent performs tabletop tasks (place, stack, draw, align) using continuous parameterized actions: \texttt{pick(x,y,z), place(x,y,z)}, and \texttt{push(x1,y1,z1,x2,y2,z2)}.
    \item \textbf{VIRL}~\cite{yang2024v}: A real-world navigation task using street-view imagery. The agent follows step-by-step instructions to reach a destination using 10 distinct actions including forward, turn, and stop.
\end{itemize}
Detailed environment statistics, prompt templates, and evaluation metrics are provided in Section 2.1 of the Supplementary Material.

\paragraph{Comparison: }
We compare \ours{} against three categories of models:
\begin{enumerate}
    \item \textbf{Proprietary Models:} Leading closed-source models including GPT-5~\cite{gpt5}, Gemini 2.5 Pro~\cite{gemini25pro}, and Claude Sonnet 4.5~\cite{claude45sonnet}.
    \item \textbf{Frozen Open-Source VLMs:} Top-performing small-scale models.
    \item \textbf{Reinforcement Learning Methods:} We compare the finetuning performance of the Qwen2.5-VL-3B backbone using standard reinforcement learning methods, including Token-PPO~\cite{schulman2017proximal,rlhf}, Turn-PPO~\cite{schulman2017proximal,chen2025reinforcement}, GRPO~\cite{guo2025deepseek}, RL4VLM~\cite{rl4vlm}, and VAGEN-Base~\cite{vagen}.
\end{enumerate}
The Qwen2.5-VL-3B backbone was chosen for RL training due to its robust architecture and extensive community support, facilitating reproducible results. Training hyperparameters for all baselines are detailed in Section 2.2 of the Supplementary Material.

\begin{table}[t]
\centering\scriptsize
\setlength{\tabcolsep}{5pt}
\caption{Comparison of \ours{} with state-of-the-art methods and models on five multi-turn decision-making benchmarks.}
\label{tab:results}
\resizebox{\linewidth}{!}{%
\begin{tabular}{l|cccccccccc|c}\toprule
    \multirow{2}{*}{{Model/Method}} & \multirow{2}{*}{\textbf{Sokoban}} & \multirow{2}{*}{\textbf{F.Lake}} & \multicolumn{2}{c}{\textbf{Navigation}} & \multicolumn{4}{c}{\textbf{Primitive Skill}} & \multicolumn{2}{c|}{\textbf{VIRL}} & \multirow{2}{*}{\textbf{Avg.}} \\
\cmidrule(lr){4-5}\cmidrule(lr){6-9}\cmidrule(lr){10-11}
&  &  & Base & CS & Place & Stack & Draw & Align & ID & OOD &  \\\midrule
\multicolumn{12}{l}{\textit{Proprietary Models}}\\\midrule
GPT-5~\cite{gpt5} & 0.7& 0.77& 0.75& 0.81& \textbf{1}& 0.63& 0& \textbf{1}&  0.06&  0.56& 0.63\\
Gemini 2.5 Pro~\cite{gemini25pro} & 0.58 & 0.78 & 0.63 & 0.63 & 0.63 & 0.63 & 0 & 0.75 &  0.16&  0.39& 0.52\\
Claude Sonnet 4.5~\cite{claude45sonnet} & 0.31 & 0.8 & 0.67 & 0.67 & 0.63 & 0.5 & 0 & \textbf{1} &  0.09&  0.44& 0.51\\\midrule
\multicolumn{12}{l}{\textit{Frozen Open-Source Models}}\\\midrule
Qwen2.5-VL-3B~\cite{qwen25vl} & 0.18& 0.13& 0.34& 0.28& 0& 0& 0& 0&  0&  0& 0.09\\
Gemma-3-4B~\cite{gemmateam2025gemma3technicalreport} &  0.07&  0.05&  0.03&  0.02&  0&  0&  0&  0&  0&  0& 0.02\\
VLM-R1-3B~\cite{shen2025vlm} & 0.16& 0.10& 0.34& 0.15& 0& 0& 0& 0&  0.03&  0& 0.08\\
Cosmos-R1-7B~\cite{cosmosr17b} & 0.14& 0.23&  0.01&  0&  0&  0&  0&  0&  0&  0& 0.04\\
Qwen3-VL-4B~\cite{qwen3vl}   &  0.27&  0.08&  0.01&  0&  0&  0&  0&  0&  0&  0& 0.04\\
\midrule
\multicolumn{12}{l}{\textit{Reinforcement Learning Methods}}\\\midrule
Token-PPO~\cite{rlhf} & 0.59& 0.72& \textbf{0.90}&  0.84& \textbf{1}& \textbf{1}& 0.99& 0.95&  0.16&  0.94& 0.81\\
Turn-PPO~\cite{chen2025reinforcement} & 0.38& 0.70& 0.78& 0.84& 0& 0& 0& \textbf{1}&  0.09& 0.94 & 0.47\\
GRPO~\cite{guo2025deepseek} & 0.20& 0.57& 0.88& 0.81& 0& 0& 0& \textbf{1}& 0.03 &0.28  &0.38 \\
RL4VLM~\cite{rl4vlm}     &  0.20&  0.35& 0.86 & 0.78 &  0&  0.95&  0&  0.98&  0.38&  \textbf{1}& 0.55\\
VAGEN-Base~\cite{vagen} & 0.61& 0.71& 0.78& 0.80& \textbf{1}& 0.88& 0.88& 0.88&  0.06 & 0.83 & 0.74\\
\rowcolor{rowrl!20}\rule{0pt}{2.3ex}\textbf{\ours{}}(Ours) & \textbf{0.83}& \textbf{0.80}& \textbf{0.90}& \textbf{0.86}& \textbf{1}& \textbf{1}& \textbf{1}& 0.99&  \textbf{0.72}&  \textbf{1}& \textbf{0.91}\\\bottomrule
\end{tabular}
}
\end{table}
\subsection{Results}
As illustrated in \cref{tab:results}, proprietary models consistently outperform open-source alternatives by a wide margin. Specifically, while all proprietary models achieve an average score exceeding 0.5 across the five benchmarks, their open-source counterparts struggle to surpass the 0.1 threshold. Notably, even with the evolution of general-purpose small models, progress remains inconsistent. For instance, while transitioning from Qwen2.5-VL-3B to Qwen3-VL-4B yields a localized improvement on Sokoban (rising from 0.18 to 0.27), this gain is overshadowed by a systemic decline across other benchmarks. Consequently, the overall average score drops from 0.09 to 0.04. This disparity implies the critical necessity of multi-turn RL training to instill robust sequential decision-making capabilities in VLMs, bridging the gap toward real-world applications.

Among reinforcement learning methods, token-wise methods such as Token-PPO achieve an average score of 0.81, and VAGEN-Base reaches 0.74, consistently outperforming their turn-wise counterparts. In contrast, turn-level methodologies such as Turn-PPO and RL4VLM exhibit diminished efficacy, yielding average scores of only 0.47 and 0.55, respectively. This decline is primarily attributed to the omission of fine-grained token-level guidance, which is indispensable for stable autoregressive generation. Furthermore, while GRPO has emerged as a cornerstone for single-turn reasoning, its inability to leverage intermediate rewards at the turn level leads to an average performance of 0.38 across the five benchmarks. 

\ours{} achieves SOTA performance across all benchmarks. Specifically, our method reaches 0.83 on Sokoban and 0.80 on Frozen Lake, while maintaining a high average score of 0.88 on Navigation. Remarkably, \ours{} achieves a near-perfect average score of 1.0 on Primitive Skill and a robust 0.86 on the VIRL task. By synergistically leveraging the \textbf{mixed advantage} for precise credit assignment and the \textbf{mixed return} for stabilized value training, our method effectively bridges the gap between high-level reasoning and low-level token generation. Notably, \ours{} achieves a remarkable average score of 0.91, significantly outperforming both token-wise and turn-wise baselines and demonstrating its robustness in complex, multi-turn interaction scenarios.

\begin{figure}[H]
    \centering
    \includegraphics[width=.8\linewidth]{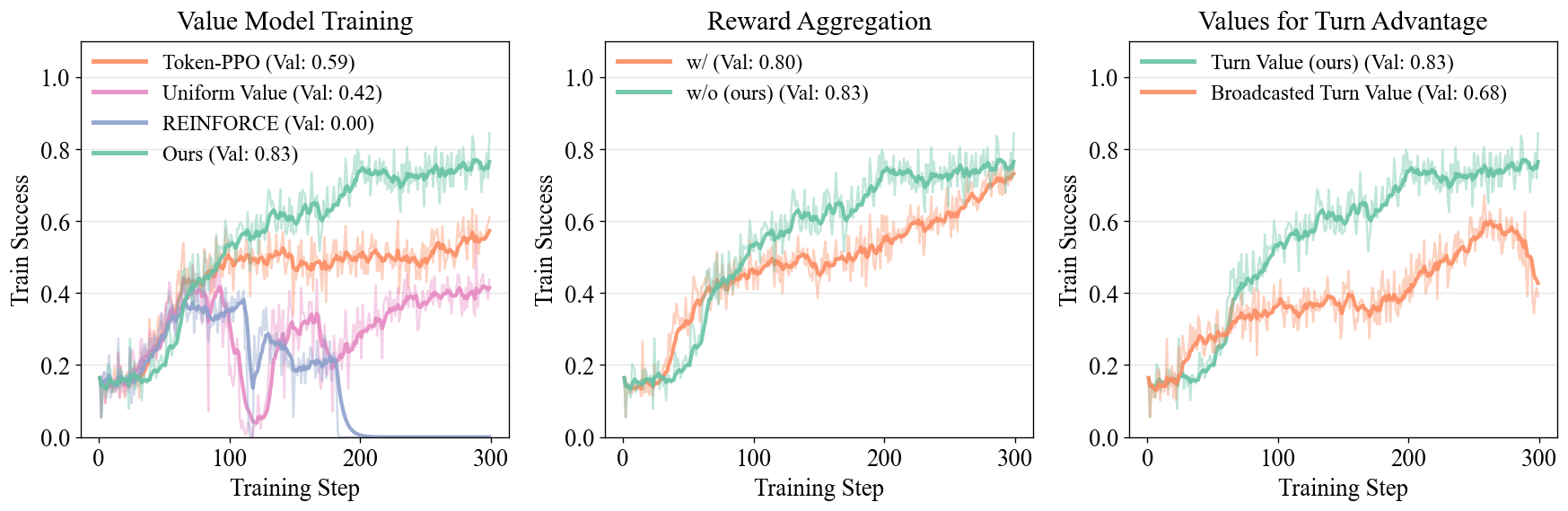}
    \caption{Key design factors for the performance of our algorithm: value model training, advantage aggregation, and value selection for constructing turn-level advantages.}
    \label{fig:key}
\end{figure}
\subsection{Key to the Success of \ours{}}\label{sec:key}
During development, we identify several critical design choices that contribute to the stability and performance of \ours{}. The diagnostic studies in this subsection are conducted on the Sokoban task.

\paragraph{The return must be precise.} 
The formulation of the mixed return in \cref{eq:mixed_return} is critical. We observe that alternative choices, such as broadcasting the turn-level return uniformly across all tokens within a turn or omitting the value model entirely (i.e., using REINFORCE), lead to training collapse, as shown in \cref{fig:key}. This observation aligns with known challenges in PPO for language models, where training the value model is difficult~\cite{shen2024advanced,sheng2024hybridflow}. The formulation of the hybrid return in \cref{eq:mixed_return} provides the precise target and necessary variance reduction to stabilize the value function $V_\psi$ throughout training.

\paragraph{Aggregation of small token-wise rewards is optional.} 
As discussed in \cref{sec:estimation}, strictly speaking, additional token-level rewards, \eg, KL-divergence penalties, should be aggregated into the turn reward when computing the turn-level advantages. This ensures that the system provides exactly the same reward from the token-wise and turn-wise RL perspectives. However, since these rewards are typically small, we show in \cref{fig:key} that aggregating them at the turn level or not leads to similar results.

\paragraph{Which value should be used for the broadcasted turn-wise advantage?}
As discussed in \cref{sec:estimation}, computing the broadcasted turn-wise advantage for intermediate tokens requires selecting an appropriate value estimate. We evaluate two strategies: 
1) using the token-specific value, which provides fine-grained, state-dependent guidance but relies on potentially noisy broadcasted intermediate value; or 
2) using the turn-level value as a constant baseline for all tokens within a turn. 
Our empirical results in \cref{fig:key} show that the theoretically grounded choice of using the turn-level value is sufficiently robust.

\paragraph{Numerical analysis on bias-variance.}
We further examine the bias-variance trade-off on Sokoban by comparing the stability of PPO and \ours{} across runs. PPO obtains a success rate of $0.5797 \pm 0.2105$, whereas \ours{} reaches $0.8222 \pm 0.0711$. This result indicates that \ours{} not only improves the average performance but also substantially reduces variance, suggesting a more stable optimization process.

\subsection{Ablation Study on $\alpha$}\label{sec:ablation_study}

\begin{table}[H]
    \centering
    \scriptsize
    \setlength{\tabcolsep}{9pt}
    \caption{Ablation study on the mixing coefficient $\alpha$.}
    \label{tab:ablationmix}\vspace{-2mm}
    \begin{tabular}{lccccc}
        \toprule
        $\alpha$& 0 & 0.25 & 0.5 & 0.75 & 1 \\
        \midrule
        Sokoban & 0.59 & 0.64 & \textbf{0.83} & 0.78 & 0.38 \\
        FrozenLake & 0.72 & 0.76 & \textbf{0.80} & 0.73 & 0.70 \\
        Primitive Skill (avg.) & 0.98 & 0.73 & \textbf{1.00} & 0.73 & 0.25 \\
        \bottomrule
    \end{tabular}\vspace{-2mm}
\end{table}

We ablate the mixing coefficient $\alpha$ used in both the hybrid advantage in \cref{eq_advantage_hybrid} and the hybrid return in \cref{eq:mixed_return}, where $\alpha$ controls the balance between turn-level and token-level signals. In this setting, $\alpha = 0$ corresponds to Token PPO, $\alpha = 1$ reduces to Turn PPO with broadcasted values, and $\alpha = 0.5$ is our default configuration. As shown in \cref{tab:ablationmix}, the default choice $\alpha = 0.5$ achieves the best performance on Sokoban, FrozenLake, and Primitive Skill. The intermediate mixtures also generally outperform the pure turn-wise setting, suggesting that combining token-level and turn-level credit assignment provides a more reliable training signal.

\subsection{Qualitative Analysis}

We compare  the frozen Qwen2.5-VL and \ours{} model to investigate their decision-making behaviors in \cref{fig:qualitative}. To improve the readability, we modify the output template and highlight the observation, reasoning, prediction, and answer here. The frozen model falls into repetitive and non-adaptive action loops. In both environments, the frozen model consistently outputs the same action across multiple turns, even when the action has no effect on the environment. In contrast, our model demonstrates an active response to environmental feedback: after an initial "Right" action yields no progress, it immediately incorporates this failure and pivots its strategy, correctly executing a "Left" move in the subsequent turn to reach the box. In the VIRL environment, our model explicitly follows the instruction and reaches the destination.

\newcommand{\qualimage}[1]{\raisebox{\dimexpr-\height+\ht\strutbox\relax}{\includegraphics[width=\linewidth]{#1}}}

\begin{figure}[H]
    \centering
    \tiny
    \setlength{\tabcolsep}{2pt}
    \begin{tabular}{p{0.48\linewidth}p{0.48\linewidth}}
        \multicolumn{1}{c}{\textbf{Frozen Qwen2.5-VL}} &
        \multicolumn{1}{c}{\textbf{\ours{}}} \\
        \begin{tabular}[t]{@{}p{0.22\linewidth}p{0.73\linewidth}@{}}
            \qualimage{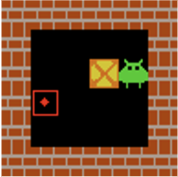} &
            \textbf{Obs.} The box is to the left of the player and the target is directly above the player.
            \textbf{Reasoning.} The player needs to push the box to the target, which is directly above the player.
            \textbf{Pred.} The player will be above the box and target, and the target and box will be at the same row.
            \textbf{Answer.} Down.
        \end{tabular}
        &
        \begin{tabular}[t]{@{}p{0.22\linewidth}p{0.73\linewidth}@{}}
            \qualimage{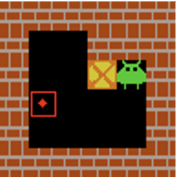} &
            \textbf{Obs.} The box is to the left of the player and the target is to the left of the player.
            \textbf{Reasoning.} The box needs to be moved to the target. The player can move right, up, or down to reach the box.
            \textbf{Pred.} The player will move \textcolor{orange}{right} to reach the box.
            \textbf{Answer.} Right.
        \end{tabular}
        \\
        \begin{tabular}[t]{@{}p{0.22\linewidth}p{0.73\linewidth}@{}}
            \qualimage{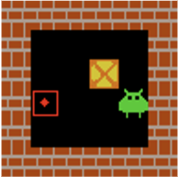} &
            \textbf{Obs.} The box is to the left of the player and the target is directly above the player.
            \textbf{Reasoning.} The player needs to push the box to the target, which is directly above the player.
            \textbf{Pred.} The player will be above the box and target, and the target and box will be at the same row.
            \textbf{Answer.} Down.
        \end{tabular}
        &
        \begin{tabular}[t]{@{}p{0.22\linewidth}p{0.73\linewidth}@{}}
            \qualimage{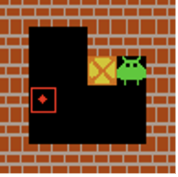} &
            \textbf{Obs.} The box is to the left of the player and the target is to the left of the player.
            \textbf{Reasoning.} The box needs to be moved to the target. The player can move left, up, or down to reach the box.
            \textbf{Pred.} The player will move \textcolor{orange}{left} to reach the box.
            \textbf{Answer.} Left.
        \end{tabular}
        \\
        \begin{tabular}[t]{@{}p{0.22\linewidth}p{0.73\linewidth}@{}}
            \qualimage{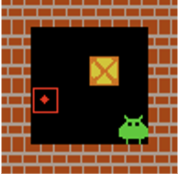} &
            \textbf{Obs.} The box is to the left of the player and the target is directly above the player.
            \textbf{Reasoning.} The player needs to push the box to the target, which is directly above the player.
            \textbf{Pred.} The player will be above the box and target, and the target and box will be at the same row.
            \textbf{Answer.} Down.
        \end{tabular}
        &
        \begin{tabular}[t]{@{}p{0.22\linewidth}p{0.73\linewidth}@{}}
            \qualimage{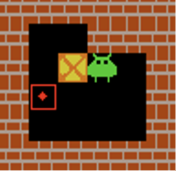} &
            \textbf{Obs.} The box is to the left of the player and the target is to the left of the player.
            \textbf{Reasoning.} The box needs to be moved to the target. The player can move left or down to reach the box.
            \textbf{Pred.} The player will move \textcolor{orange}{left} to reach the box.
            \textbf{Answer.} Left.
        \end{tabular}
        \\[0.15em]
        \multicolumn{2}{c}{\makebox[0.48\linewidth]{\leaders\hbox{--\hspace{0.35em}}\hfill}} \\
        \\[-0.65em]
        \begin{tabular}[t]{@{}p{0.22\linewidth}p{0.73\linewidth}@{}}
            \qualimage{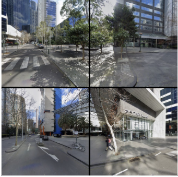} &
            \textbf{Obs.} The scene contains buildings with glass facades, a blue sign, trees, and a sign for ``The Langham, Melbourne.''
            \textbf{Reasoning.} The instruction requires facing northeast, moving forward until the next intersection, and then turning left to face northeast.
            \textbf{Pred.} Continue forward to the next intersection.
            \textbf{Answer.} forward.
        \end{tabular}
        &
        \begin{tabular}[t]{@{}p{0.22\linewidth}p{0.73\linewidth}@{}}
            \qualimage{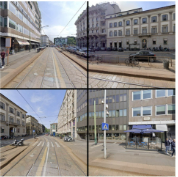} &
            \textbf{Obs.} I am facing east.
            \textbf{Reasoning.} \textcolor{orange}{First instruction is to turn right} to face east.
            \textbf{Pred.} I will turn right.
            \textbf{Answer.} turn\_direction(east), forward().
        \end{tabular}
        \\
        \begin{tabular}[t]{@{}p{0.22\linewidth}p{0.73\linewidth}@{}}
            \qualimage{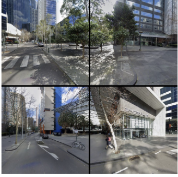} &
            \textbf{Obs.} The scene again shows the same buildings, trees, and road context.
            \textbf{Reasoning.} The model repeats that it should face northeast, move forward until the next intersection, and then turn left.
            \textbf{Pred.} Continue forward to the next intersection.
            \textbf{Answer.} forward.
        \end{tabular}
        &
        \begin{tabular}[t]{@{}p{0.22\linewidth}p{0.73\linewidth}@{}}
            \qualimage{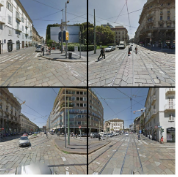} &
            \textbf{Obs.} I am facing east.
            \textbf{Reasoning.} \textcolor{orange}{I have turned right}, now I need to move forward.
            \textbf{Pred.} I will move forward.
            \textbf{Answer.} forward(), turn\_direction(north), forward().
        \end{tabular}
        \\
        \begin{tabular}[t]{@{}p{0.22\linewidth}p{0.73\linewidth}@{}}
            \qualimage{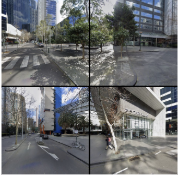} &
            \textbf{Obs.} The scene still contains the same road and building context.
            \textbf{Reasoning.} The model repeats the previous plan to turn left, move forward until reaching the next intersection, and then continue.
            \textbf{Pred.} I should continue forward to the next intersection.
            \textbf{Answer.} forward().
        \end{tabular}
        &
        \begin{tabular}[t]{@{}p{0.22\linewidth}p{0.73\linewidth}@{}}
            \qualimage{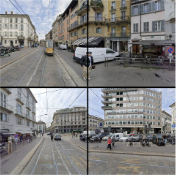} &
            \textbf{Obs.} \textcolor{orange}{I am facing north}.
            \textbf{Reasoning.} I have turned right, now I need to move forward.
            \textbf{Pred.} I will move forward.
            \textbf{Answer.} forward().
        \end{tabular}
    \end{tabular}
    \caption{Qualitative examples on Sokoban (up) and VIRL (bottom) with readable template. The frozen model repeatedly predicts ineffective actions or repeats a static plan, whereas \ours{} updates its actions after receiving feedback and uses previous actions and observations to maintain multi-turn consistency.}
    \label{fig:qualitative}
\end{figure}

\section{Conclusion}
In this paper, we study the problem of training Large Vision-Language Models for multi-turn decision-making tasks. In these tasks, at each turn, the model must reason over intermediate visual states and make decisions based on its internal reasoning. We propose \ours{}, an agentic reinforcement learning algorithm employing the actor-critic framework. Specifically, we introduce a hybrid advantage that combines a lightly discounted turn-level advantage for decision making with a token-level advantage that better supports general language generation, while using a unified critic for efficient training. We theoretically prove the validity of this hybrid formulation and show that the resulting framework achieves both low variance and low bias. In experiments with 3B-scale VLMs across five decision-making benchmarks, our method consistently outperforms prior approaches. We further provide quantitative and qualitative analyses to enable a deeper understanding of our algorithm.

\bibliographystyle{assets/arxiv}
\bibliography{references}
\newpage
\appendix
\section{Theoretical Analysis for Section 4}

\subsection{Formulations}

We illustrate the formulations in \cref{sec:turn_wise_and_token_wise_surrogate}.
For ease of reading, we will write the equations and theorems again as needed and give appropriate reference to the main text.

\subsubsection{Turn-wise POMDP}
\paragraph{Objective Function}
The turn-wise objective can be expressed as follows.

\begin{equation}
J^{\rm turn}(\theta)
=  \mathbb{E}_{\pi_\theta}
\left[   
    \sum_{t=0}^{ T-1 } {\boldsymbol\gamma}_t^t \boldsymbol{r}_t
\right]
\end{equation}

\paragraph{Surrogate Function}
The corresponding surrogate objective is

\begin{equation}\label{app:eq_surrogate_turn}
L^{\mathrm{turn}} (\theta )=\mathbb{E}\left[\sum _{t=0}^{T-1}\sum _{i=0}^{I_{t} -1}\log \pi _{\theta }\left( a_{t}^{i} |\tau _{t}^{i}\right)\boldsymbol{A}^{\pi _{\theta }} (\boldsymbol{\tau }_{t} ,\boldsymbol{a}_{t} )\right]
\end{equation}

\subsubsection{Token-wise POMDP}

\paragraph{Objective Function}
The token-wise objective can be expressed in two equivalent formulations, depending on the indexing scheme.

In the \emph{local indexing formulation}, the reward \( r_t^i \) is indexed {locally within each turn}, where \( t \) denotes the {turn index} and \( i \) denotes the {intra-turn token index}. This means that rewards are treated as belonging to {individual turns}, with each turn having its own sequence of rewards, indexed independently within that turn.

\begin{align}
\text{Local indexing:} \quad & J^{\mathrm{token}}(\theta) \triangleq \mathbb{E}_{\pi_{\theta}}\left[\sum_{t=0}^{T-1} \sum_{i=0}^{I_t - 1} \gamma^{i + I_{< t}} r_t^i \right] \label{eq_token_wise_objective_local}
\end{align}

In the \emph{global indexing formulation}, the reward \( r^{(i)} \) is indexed globally, where \( i \) represents the {global token index} across the entire trajectory. The discount factor is applied uniformly across all tokens in the trajectory.

\begin{align}
\text{Global indexing:} \quad & J^{\mathrm{token}}(\theta) \triangleq \mathbb{E}_{\pi_{\theta}}\left[\sum_{i=0}^{I_{< T} - 1} \gamma^i r^{(i)}\right]
\label{eq_token_wise_objective_global}
\end{align}

\paragraph{Surrogate Function}

The surrogate objective function can be expressed as follows.

For the \emph{local indexing formulation}, where rewards are indexed within each turn:
\begin{equation}\label{app:eq_surrogate_token}
\text{Local indexing:} \quad
L^{\mathrm{token}} (\theta )=\mathbb{E}\left[\sum _{t=0}^{T-1}\sum _{i=0}^{I_{t} -1}\log \pi _{\theta }\left( a_{t}^{i} |\tau _{t}^{i}\right) A^{\pi _{\theta }}\left( \tau _{t}^{i} ,a_{t}^{i}\right)\right]
\end{equation}

For the \emph{global indexing formulation}, where rewards are indexed globally across the entire trajectory:
\begin{equation}\label{eq_surrogate_token_global_index}
\text{Global indexing:} \quad
L^{\mathrm{token}} (\theta ) = \mathbb{E}\left[\sum _{i=0}^{I_{< T} -1}\log \pi _{\theta }\left( a^{(i)} |\tau^{(i)}\right) A^{\pi _{\theta }}\left( \tau^{(i)}, a^{(i)} \right)\right]
\end{equation}

\subsection{Theorems and Proofs}

\begin{lemma}[Policy Gradient Theorem]
The gradient of the objective $J(\theta)$ satisfies
$$\nabla J(\theta) = \nabla L_\text{actor}( \theta )  $$
\end{lemma}
\begin{proof}
The result follows directly from the policy gradient theorem
\cite{sutton1999policy}.
\end{proof}

\begin{reptheorem}{theorem_gradient}
Let the discount factors of the turn-wise and token-wise formulations satisfy $\gammaTurn_t = \gammaToken^{I_t}$. Then, for any $\alpha \in [0,1]$ in \cref{eq_surrogate_hybrid}, we have
\[
\nabla J(\theta)
=
\nabla L^{\mathrm{hybrid}}(\theta)
=
\nabla L^{\mathrm{turn}}(\theta)
=
\nabla L^{\mathrm{token}}(\theta).
\]
\end{reptheorem}

\begin{proof}
We first prove that $\nabla L^{\mathrm{turn}}(\theta)=\nabla J(\theta)$. 
As shown in \cref{eq_surrogate_turn},
\[
 L^{\mathrm{turn}}(\theta)
 =
 L_{\text{actor}}(\theta),
\]
which directly yields the result.

Next, we prove that $\nabla L^{\mathrm{token}}(\theta)=\nabla J(\theta)$. 
Note that the objective in \cref{eq_objective} can be rewritten by indexing tokens, see \cref{eq_token_wise_objective_local,eq_token_wise_objective_global}.
Following the same derivation as in \cref{eq_policy_gradient}, we can obtain the token-wise policy gradient based on \cref{eq_token_wise_objective_global}:
\begin{align}
& \nabla J(\theta)  \notag
\\
=& \nabla J^{\rm token}(\theta) \notag
\\
= &
\nabla \mathbb{E}\left[\sum _{i=0}^{I_{< T} -1}\log \pi _{\theta } (a^{( i)} |\tau ^{( i)} )A^{\pi _{\theta }} (\tau ^{( i)} ,a^{( i)} )\right] \notag
\\
= &
\nabla 
\mathbb{E}\left[\sum _{t=0}^{T-1}\sum _{i=0}^{I_{t} -1}\log \pi _{\theta }\left( a_{t}^{i} |\tau _{t}^{i}\right) A^{\pi _{\theta }}\left( \tau _{t}^{i} ,a_{t}^{i}\right)\right] 
\label{eq_temp_asjuqjushisoas}
\\
= &
\nabla 
L^{\rm token}(\theta) \notag
\end{align}
Here, \cref{eq_temp_asjuqjushisoas} is obtained by reindexing the trajectory from token-only indices to mixed turn–token indices.

Finally, we derive the hybrid policy gradient:
\begin{align*}
 & \nabla J(\theta )\\
= & \alpha \nabla L^{\mathrm{turn}} (\theta )+(1-\alpha)\nabla L^{\mathrm{token}} (\theta )\\
= & \nabla \mathbb{E}\left[\sum _{t=0}^{T-1}\sum _{i=0}^{I_{t} -1}\log \pi _{\theta }\left( a_{t}^{i} |\tau _{t}^{i}\right)\left( \alpha \boldsymbol{A}^{\pi _{\theta }}\left(\boldsymbol{\tau }_{t} ,\boldsymbol{a}_{t}\right) +( 1-\alpha ) A^{\pi _{\theta }}\left( \tau _{t}^{i} ,a_{t}^{i}\right)\right)\right]\\
= & \nabla L^{\mathrm{hybrid}} (\theta )
\end{align*}
\end{proof}

\begin{reptheorem}{thm:unified_value_function}
Let the discount factors of the turn-wise and token-wise formulations satisfy $\gammaTurn_t = \gammaToken^{I_t}$. Then the corresponding value functions are consistent at the last token of the trajectory, that is,
\begin{equation}\label{app:eq:value_function_relation}
    \boldsymbol{V}^\pi(\boldsymbol{\tau}_t) = V^\pi(\tau_t^{I_t}) \quad \text{for any turn } t 
    .
\end{equation}
\end{reptheorem}
\begin{proof}

We begin by expressing the turn-wise value function at time step $t$ as follows:
\begin{align}
\boldsymbol{V}^{\pi } (\boldsymbol{\tau }_{t} ) 
& \triangleq \mathbb{E}_{\pi }\left[\sum _{t'=0}^{T-t}\boldsymbol{\gamma }_{t'}^{t'}\boldsymbol{r}_{t+t'} \mid \boldsymbol{\tau }_{t}\right]
\notag
\\
 & 
 =\mathbb{E}_{\pi }\left[\sum _{t'=0}^{T-t} \gamma ^{I_{t'} t'}\boldsymbol{r}_{t+t'} \mid \boldsymbol{\tau }_{t}\right]
 \label{eq_temp_aksjqihwuhhqha}
 \\
 & 
 =\mathbb{E}_{\pi }\left[\sum _{t'=0}^{T-t}\sum _{i=0}^{I_{t} -1}\mathbf{1}_{i=I_{t}} \gamma ^{I_{t'} t'}\boldsymbol{r}_{t+t'} \mid \boldsymbol{\tau }_{t}\right]
 \notag
 \\
 & 
 =\mathbb{E}_{\pi }\left[\sum _{t'=0}^{T-t}\sum _{i=0}^{I_{t} -1} \gamma ^{i+I_{< t'}} r_{t+t'}^{i} \mid \tau _{t}^{I_{t}}\right]
 \label{eq_temp_aqwsssqquhhqha}
 \\
 & =V^{\pi } (\tau _{t}^{I_t} )
 \notag
\end{align}
where $\mathbf{1}$ is the indicator function.
\cref{eq_temp_aksjqihwuhhqha}  follows from the assumption that the discount factor $\boldsymbol{\gamma}_{t'}=\gamma^{I_{t'}}$.
\cref{eq_temp_aqwsssqquhhqha} utilizes the definition of $r_t^i$ as given in the main text.

\end{proof}

\begin{lemma}[Bias Bound for the $\lambda$-Return Advantage Estimator{, adapted from \cite{kearns2000bias}}]
\label{theorem_bias_general}
Let $\boldsymbol{e}_{\max} \triangleq \| \hat{\boldsymbol{V}} - \boldsymbol{V}^{\pi} \|_{\infty} $, then
\begin{equation}
    \mathrm{bias}(\hat{\boldsymbol{A}}_{t}) \leq {\left( 1+\frac{\boldsymbol{\gamma}( 1-\lambda )}{1-\boldsymbol{\gamma} \lambda }\right)} \boldsymbol{e}_{\max}
\end{equation}
\end{lemma}

\begin{proof}
The result follows from the bias analysis of $\lambda$-returns in \cite{kearns2000bias}.
\end{proof}

\begin{reptheorem}[Bias Comparison]{theorem_bias_comparision}
Let $e_{\max} \triangleq \| \hat{V} - V^{\pi} \|_{\infty}$, and $I=\min_t I_t$, 
then
\begin{fontsize}{9}{11}
\begin{equation*}
\mathrm{bias}\!\left(\hat{\boldsymbol{A}}_{t}\right)
\le
\underbrace{
\left( 1+\frac{\gamma ^{I}( 1-\lambda )}{1-\gamma ^{I} \lambda }\right)
}_{b_{\mathrm{turn}}}
e_{\max},
\ 
\mathrm{bias}\!\left(\hat{A}_{t}^{i}\right)
=
\underbrace{\left(
1+\frac{\gamma(1-\lambda)}{1-\gamma\lambda}
\right)}_{b_{\mathrm{token}}}
e_{\max},
\ 
b_{\mathrm{turn}} \leq b_{\mathrm{token}}.
\end{equation*}
\end{fontsize}
\end{reptheorem}

\begin{proof}

We begin by analyzing the bias of the advantage estimator for the turn-wise trajectory:
\begin{align}
\mathrm{bias} (\hat{\boldsymbol{A}}_{t} ) & \leq \max_t \left( 1+\frac{\boldsymbol{\gamma }_t (1-\lambda )}{1-\boldsymbol{\gamma }_t \lambda }\right)\boldsymbol{e}_{\max} 
\notag
\\
 & =\left( 1+\frac{{\gamma }^{\min_t I_t} (1-\lambda )}{1-{\gamma }^{\min_t I_t} \lambda }\right)\boldsymbol{e}_{\max}  \notag
 \\
 & \leq \left( 1+\frac{{\gamma }^{I} (1-\lambda )}{1-{\gamma }^{I} \lambda }\right) e_{\max} \label{eq_temp_ajsjqhijnnains}
\end{align}
In \cref{eq_temp_ajsjqhijnnains}, we applied the fact that
$
\boldsymbol{e}_{\max} \triangleq \| \hat{\boldsymbol{V}} -\boldsymbol{V}^{\pi } \| _{\infty } \leq \| \hat{V} -V^{\pi } \| _{\infty } =e_{\max}
$, 
since the turn-wise trajectory space is a subset of the token-wise one.

The bias bound $\mathrm{bias}\left(\hat{A}_{t}^{i}\right)$ can be derived in a similar manner, following \cref{theorem_bias_general}.

Given that $0 < \gamma < 1$, we observe that $\gamma^{I} (1 - \lambda) \leq \gamma (1 - \lambda)$. Moreover, since $I > 1$, it holds that $1 - \gamma^{I} \lambda \geq 1 - \gamma \lambda$. Therefore, we conclude that $b_{\mathrm{turn}} \leq b_{\mathrm{token}}$, which completes the proof.

\end{proof}

\begin{lemma}[Variance Bound for $\lambda$-Return Advantage Estimator, adapted from \cite{kearns2000bias}]
\label{theorem_variance_general}
Let $\boldsymbol{\sigma }_{\max} \triangleq \sup \mathrm{var} (\boldsymbol{\delta }_{t} ) $, then
\begin{equation}\label{eq_ijjsihqjsiojq}
   \mathrm{var}(\hat{\boldsymbol{A}}_{t}) \leq \left(\frac{1}{1-\boldsymbol{\gamma } \lambda }\right)^{2} \boldsymbol{\delta }_{\max}
\end{equation}
\end{lemma}
\begin{proof}
The result follows from the variance analysis of $\lambda$-returns in \cite{kearns2000bias}.
\end{proof}

\begin{reptheorem}[Variance Comparison]{theorem_variance_comparision}
Let $\sigma_{\max} \triangleq \sup \mathrm{var}(\delta^{(i)})$, and $I=\max_t I_t$.
Then
\begin{fontsize}{9}{11}
\begin{equation*}
\mathrm{var}\!\left(\hat{\boldsymbol{A}}_{t}\right)
\le
\underbrace{\left(
\frac{1-\gamma^{I}}
{(1-\gamma)(1-\gamma^{I}\lambda)}
\right)^{2}}_{v_{\mathrm{turn}}}
\sigma_{\max},
\ 
\mathrm{var}\!\left(\hat{A}_{t}^{i}\right)
\le
\underbrace{\left(
\frac{1}{1-\gamma\lambda}
\right)^{2}}_{v_{\mathrm{token}}}
\sigma_{\max},
\ 
v_{\mathrm{turn}} > v_{\mathrm{token}}.
\end{equation*}
\end{fontsize}
\end{reptheorem}

\begin{proof}
We begin by expressing the turn-wise temporal difference as:

\begin{equation}\label{eq_temp_jsnjnqijnsnlkna}
\begin{aligned}
\boldsymbol{\delta}_{t}
&= \boldsymbol{r}_{t} + \boldsymbol{\gamma \hat{V}}(\boldsymbol{\tau}_{t+1}) - \hat{V}(\boldsymbol{\tau}_{t})
 \\
&= r_{t}^{0} + \boldsymbol{\gamma \hat{V}}\!\left(\tau_{t+1}^{1}\right) - \hat{V}\!\left(\tau_{t}^{0}\right)
  + \gamma r_{t}^{1} + \boldsymbol{\gamma^{2}\hat{V}}\!\left(\tau_{t+1}^{2}\right)
  - \gamma \hat{V}\!\left(\tau_{t}^{1}\right) + \cdots
   \\
&= \sum_{i=0}^{I_{t}-1} \gamma^{i}
   \left(
   r_{t}^{i}
   + \boldsymbol{\gamma \hat{V}}\!\left(\tau_{t}^{i+1}\right)
   - \hat{V}\!\left(\tau_{t}^{i}\right)
   \right)
    \\
&= \sum_{i=0}^{I_{t}-1} \gamma^{i} \delta_{t}^{i}
\end{aligned}
\end{equation}

Now, considering the supremum of turn-wise variances, we obtain:

\begin{align}
\boldsymbol{\sigma}_{\max}
&\triangleq \sup \mathrm{var}(\boldsymbol{\delta}_{t})
     \notag
\\
&\leq \sup \mathrm{var}\!\left(\sum_{i=0}^{I_{t}-1} \gamma^{i} \delta_{t}^{i}\right) 
    \label{eq_temp_hjqnsniqhijq}
\\
&\leq \max_t \left(\frac{1-\gamma^{I_t}}{1-\gamma}\right)^{2}
   \sup \mathrm{var}(\delta_{t}^{i}) 
    \notag
\\
&\leq  \left(\frac{1-\gamma^{ I}}{1-\gamma}\right)^{2}
   \sup \mathrm{var}(\delta_{t}^{i}) 
    \notag
\\
&= \left(\frac{1-\gamma^{I}}{1-\gamma}\right)^{2}
   \sup \mathrm{var}(\delta^{(i)}) 
    \notag
\\
&= \left(\frac{1-\gamma^{I}}{1-\gamma}\right)^{2} \sigma_{\max} \notag
\end{align}

\cref{eq_temp_hjqnsniqhijq} follows directly from \cref{eq_temp_jsnjnqijnsnlkna}.

Thus, we derive the variance bound for $\hat{\boldsymbol{A}}_{t}$:

\begin{align}
\mathrm{var} (\hat{\boldsymbol{A}}_{t} )\leq \left(\frac{1}{1-\boldsymbol{\gamma } \lambda }\right)^{2}\boldsymbol{\sigma }_{\max} \leq \left(\frac{1-\gamma ^{I}}{(1-\gamma )(1-\gamma ^{I} \lambda )}\right)^{2} \sigma _{\max}
\end{align}

Next, we can derive the token-wise variance bound $\mathrm{var}\left(\hat{A}_{t}^{i}\right)$ in a similar manner following \cref{theorem_variance_general}.

Finally, we prove that $v_{\mathrm{turn}}  >v_{\mathrm{token}}$. To do so, we compare the following two expressions under the conditions $0<\gamma,\lambda<1$ and $I>1$:

\[
\frac{1-\gamma ^{I}}{(1-\gamma)\left(1-\gamma^{I}\lambda\right)}
\quad \text{and} \quad
\frac{1}{1-\gamma\lambda}
\]

Since the denominators are positive, we compare the cross products:

\begin{align*}
\frac{1-\gamma ^{I}}{(1-\gamma)\left(1-\gamma^{I}\lambda\right)}
>
\frac{1}{1-\gamma\lambda}
\quad
\Longleftrightarrow
\quad
(1-\gamma^{I})(1-\gamma\lambda)
>
(1-\gamma)(1-\gamma^{I}\lambda).
\end{align*}

Expanding both sides:

\begin{align*}
(1-\gamma^{I})(1-\gamma\lambda)
&=1-\gamma\lambda-\gamma^{I}+\gamma^{I+1}\lambda, \\
(1-\gamma)(1-\gamma^{I}\lambda)
&=1-\gamma-\gamma^{I}\lambda+\gamma^{I+1}\lambda.
\end{align*}

Taking the difference:

\begin{align*}
&(1-\gamma^{I})(1-\gamma\lambda)
-
(1-\gamma)(1-\gamma^{I}\lambda) \\
&= \gamma-\gamma\lambda+\gamma^{I}\lambda-\gamma^{I} \\
&=(1-\lambda)(\gamma-\gamma^{I}).
\end{align*}

Since $0<\gamma<1$ and $I>1$, it follows that $\gamma^{I}<\gamma$, and since $1-\lambda>0$, we conclude that:

\[
(1-\lambda)(\gamma-\gamma^{I})>0.
\]

Therefore, we have:

\[
\frac{1-\gamma ^{I}}{(1-\gamma)\left(1-\gamma^{I}\lambda\right)}
>
\frac{1}{1-\gamma\lambda}.
\]

Thus, we confirm that $v_{\mathrm{turn}} > v_{\mathrm{token}}$.

\end{proof}

\section{Experiment Details}
\subsection{Environment}

In the rollout stage, we first input the system prompt and the initial observation into the VLM. The VLM then generates the thinking tokens and action tokens. The environment executes the action tokens and returns the reward along with a step penalty.

In the next turn, we provide the previous chat history together with the new observation as a multi-turn conversation input to the VLM. In other words, the input string for turn $i$ after applying the chat template is
\begin{lstlisting}[language=Python]
"""
<|im_start|>system
You are a ... 
<|im_end|>
<|im_start|>user
The initial observation is: ...
<|im_end|>
<|im_start|>assistant
<think> ... </think><action> ... </action>
<|im_end|>
...
<|im_start|>user
The new observation is: ...
<|im_end|>
<|im_start|>assistant
"""
\end{lstlisting}
For all environments, we adopt the grounding-world-modeling answer format proposed in \cite{vagen}.

\cref{tab:env_stats} summarizes the statistics of the environments. Below, we provide a detailed description of each environment.
\begin{table}[t]
    \centering\scriptsize
    \caption{Statistics of the environments. }
    \label{tab:env_stats}
    \begin{tabular}{lccccc}
        \toprule
        \textbf{Name} & \textbf{\#Actions} & \textbf{Max Turns}& \textbf{Train Size} & \textbf{Test Size} & \textbf{Eval Tasks}\\
        \midrule
        Sokoban & 4 & 3 & 10K & 128 & Base \\
        FrozenLake & 4 &3  & 10K & 128 & Base \\
        Navigation & 8 & 4 &  10K& 128 $\times$ 2& Base, Common Sense \\
        Primitive Skill & - & 3 & 10K $\times$ 4 & 128 $\times$ 4& Place, Stack, Draw, Align \\
        VIRL & 10 & 7 & 1K & 32+18& In-Domain, Out-of-Domain \\\bottomrule
        \end{tabular}
\end{table}
\begin{table}\centering\small
    \caption{Details of the reward. }\label{tab:env_reward}
    \begin{tabular}{lccccc}                                                                       
        \toprule                                   
        \textbf{Name} & \textbf{Format} & \textbf{Step Penalty} & \textbf{Progress} &
    \textbf{Success} \\
        \midrule
        Sokoban      & +0.5 & $-0.1$/step & +1 on target, $-1$ off target & +10 \\
        FrozenLake   & +0.5 & -          & -              & +10  \\
        Navigation   & +0.5 & -          & -               & +10  \\
        Primitive Skill & +0.5 & $-0.1$/step & +2 / stage & +10 \\
        VIRL         & +0.3 & -          & +3.0/keypoint, +0.8/correct action    & +10   \\
        \bottomrule
    \end{tabular}
\end{table}

\paragraph{Sokoban} 
We adopt the Sokoban environment originally introduced in \cite{sokoban} and implemented by \cite{vagen}. The goal of the task is to move the player so that the boxes are pushed to their target positions. The action space is 
\begin{lstlisting}[language=Python]
    Left # Move the player to the left by one step
    Down # Move the player down by one step
    Right # Move the player right by one step
    Up # Move the player up by one step
\end{lstlisting} 
The environment executes any valid action until the player reaches the target position or the maximum number of turns (3) is reached. At each turn, the environment returns a format reward and a step penalty. At the end of the episode, the environment returns a success reward to the agent. The reward details are provided in \cref{tab:env_reward}. 
During evaluation, we report the agent’s success rate. 
The training set contains 10,000 episodes, and the test set contains 128 episodes, both of which are generated by the same \texttt{gym} environment with different random seeds. Below, we show an example of the first turn of the Sokoban environment in both the training and evaluation stages.

\begin{academicprompt}{Sokoban Example}
    \textbf{System Prompt} \\[1ex]
    You are a Sokoban solver.\\
\textcolor{promptbrown}{Sokoban Quick Guide}\\
Goal: Push all boxes onto targets.\\
Symbols:\\
If there is a image provided, the red square with center point is the target, the yellow squre with red cross is the box, the green monster is the player.
If there is no image provided, the symbols are:\\
\# Wall | \_ Floor | O Target | X Box | P You | $\surd$ Box on Target | S You on Target\\
Rules:\\
1. Push boxes (can't pull).\\
2. Avoid walls.\\
\textcolor{promptbrown}{Actions you can take}: Left, Down, Right, Up.\\
You can take up to 3 action(s) at a time, separated by ,.\\
You should first give the description of your observation, then your reasoning, then predict the next state, and finally your answer.\\
For the content you provide within the `<observation>' and `<prediction>' tags, you must strictly describe the relative position of the `target'(red square with center point) and any visible `box'(yellow squre with red cross) objects **relative to the player**.  Only black positions are movable positions.  Use ONLY the terms `above', `below', `left', `right', or `same' for describing these relationships. Do not repeat previous generations. ONLY look at the most recent image.\\
\textcolor{promptbrown}{Your response should be in the format of:}\\
<think><observation> ... </observation><reasoning> ... </reasoning><prediction> ... </prediction></think><answer> ... </answer>\\
\textcolor{promptbrown}{e.g.} <think><observation>The box is below and to right of the player and the target is below and to the right of the player</observation><reasoning>I need to go right then go down to push the box down to the target</reasoning><prediction>The player will be above the box and target, the target and box will be at the same row.</prediction></think><answer>Right,Down</answer> 
    \promptdashline
    \begin{minipage}[c]{0.80\linewidth}
        \textbf{User Prompt} \\[1ex]
        {[Initial Observation]:} \\
        <image>\\
        Decide your next action(s).
    \end{minipage}%
    \hspace{1em}%
    \begin{minipage}[c]{0.15\linewidth}
        \centering
        \includegraphics[width=\linewidth]{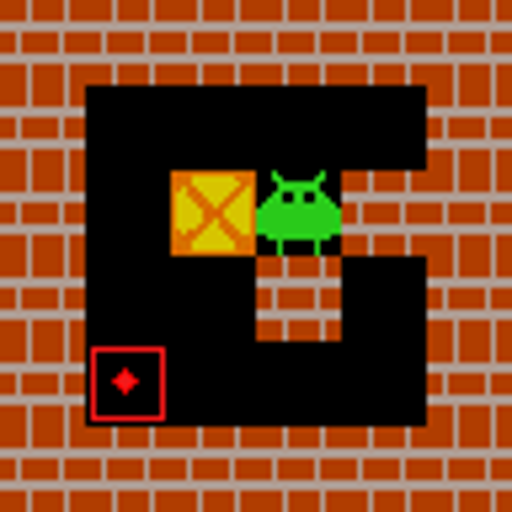}
    \end{minipage}
    \promptsolidline
    \textcolor{promptblue}{%
    \textbf{Model Output} \\[1ex]
    <think>The player is above a wall, and there is a yellow box with a red cross below the player. The target is directly above the box
<reasoning>The player needs to move to the left to push the box down to the target. The wall blocks movement directly right.</reasoning>
<prediction>Left,Down</prediction></think>
<answer>Left,Down</answer>
    }\\
    \promptsolidline
    \begin{minipage}[c]{0.80\linewidth}
        \textbf{User Prompt} \\[1ex]
        After your answer, the extracted valid action is [].\\
After that, the observation is:\\
<image>\\
Decide your next action(s).
    \end{minipage}
    \hspace{1em}%
    \begin{minipage}[c]{0.15\linewidth}
        \centering
        \includegraphics[width=\linewidth]{images/env_example/sokoban.png}
    \end{minipage}
    \promptsolidline
    ...
\end{academicprompt}

\paragraph{FrozenLake} 
We adopt the FrozenLake environment originally introduced in \cite{frozenlake} and implemented by \cite{vagen}. The goal of the task is to navigate from the start position to the goal while avoiding holes. The action space is
\begin{lstlisting}[language=Python]
    Left # Move the player to the left by one step
    Down # Move the player down by one step
    Right # Move the player right by one step
    Up # Move the player up by one step
\end{lstlisting}
The environment executes any valid action until the player reaches the goal or the maximum number of turns (3) is reached. At each turn, the environment returns a format reward. At the end of the episode, the environment returns a success reward to the agent. The reward details are provided in \cref{tab:env_reward}.
During evaluation, we report the agent’s success rate. 
The training set contains 10,000 episodes, and the test set contains 128 episodes, both of which are generated by the same \texttt{gym} environment with different random seeds. Below, we show an example of the first turn of the FrozenLake environment in both the training and evaluation stages.

\begin{academicprompt}{FrozenLake Example}
    \textbf{System Prompt} \\[1ex]
    You are a FrozenLake solver. \\
\textcolor{promptbrown}{FrozenLake Quick Guide}\\
Goal: Reach the goal (G).\\
Symbols (If image is provided there are no symbols):
\_ Frozen | O Hole | G Goal | P Player | X Player fell into hole | $\surd$  Player on goal\\
Rules:\\
1. Avoid falling into holes.\\
2. Frozen tiles are slippery, you may move perpendicular to your intended direction.\\
\textcolor{promptbrown}{Actions you can take}: Left, Down, Right, Up. \\
You can take up to 3 action(s) at a time, separated by ,.\\
You should first describe the observation, then your reasoning, then predict the next state, and finally your answer.\\
For the content you provide within the `<observation>' and `<prediction>' tags, you must strictly describe the relative position of the `target' (the gift box) and any visible `hole' (blue circles) objects **relative to the player**. Your description/prediction must include **both** a vertical and a horizontal directional relationship for each object. Use ONLY the terms `above', `below', `left', `right', or `same' for describing these relationships.\\
\textcolor{promptbrown}{Your response should be in the format of:}\\
<think> <observation> ... </observation> <reasoning> ... </reasoning> <prediction> ... </prediction> </think> <answer> ... </answer>\\
\textcolor{promptbrown}{e.g.} <think><observation>The player is above and on the right side of target. There is a hole below and at the left of the player</observation><reasoning>I should go down twice first to reach th same row as the target</reasoning><prediction>The player will be in the same row and to the right of the target.</prediction></think><answer>Down,Down</answer>\\
\promptdashline
\begin{minipage}[c]{0.80\linewidth}
    \textbf{User Prompt} \\[1ex]
    {[Initial Observation]:} \\
    <image>\\
    Decide your next action(s).\\
\end{minipage}
\hspace{1em}%
\begin{minipage}[c]{0.15\linewidth}
    \centering
    \includegraphics[width=\linewidth]{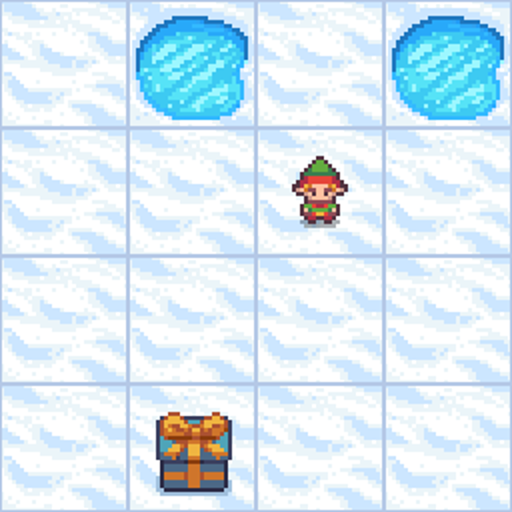}
\end{minipage}
\promptsolidline
\textcolor{promptblue}{
    \textbf{Model Output} \\[1ex]
    <think><observation>The player is to the right of the goal (G), and there is a hole directly above the player.</observation><reasoning>The player should first move to the same row as the goal to potentially reach it without falling into the hole.</reasoning><prediction>The player should move down, then to the right.</prediction></think><answer>Down, Right</answer>
}
\promptsolidline
\begin{minipage}[c]{0.80\linewidth}
    \textbf{User Prompt} \\[1ex]
    After your answer, the extracted valid action is [Down, Right].\\
After that, the observation is:\\
<image>\\
Decide your next action(s).
    \end{minipage}
    \hspace{1em}%
    \begin{minipage}[c]{0.15\linewidth}
        \centering
        \includegraphics[width=\linewidth]{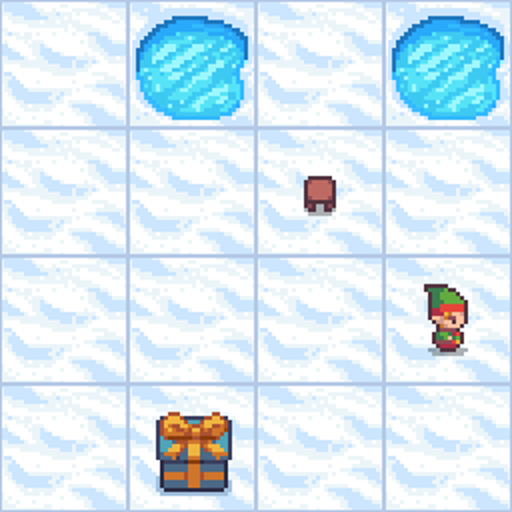}
    \end{minipage}
    \promptsolidline
    ...
\end{academicprompt}

\paragraph{Navigation} 
We adopt the Navigation environment originally introduced in \cite{yang2025embodiedbench} and implemented by \cite{vagen}. The goal of the task is to navigate to the target within 1.5 meters using an 8-action space related to movement, rotation, and camera control. The action space is
\begin{lstlisting}[language=Python]
    moveahead # Move forward by some distance 
    moveback # Move backward by some distance 
    moveright # Move rightward by some distance 
    moveleft # Move leftward by some distance 
    rotateright # Rotate to the right by 90 degrees 
    rotateleft # Rotate to the left by 90 degrees 
    lookup # Tilt the camera upward by 30 degrees 
    lookdown # Tilt the camera downward by 30 degrees 
\end{lstlisting}
The environment executes any valid action until the agent reaches the target or the maximum number of turns (4) is reached. At each turn, the environment returns a format reward. At the end of the episode, the environment returns a success reward to the agent. The reward details are provided in \cref{tab:env_reward}. 

During evaluation, we report the agent’s success rate on two tasks: Base and Common Sense. In the Base task, the agent is given a direct instruction to navigate to the target, which matches the training setting. In the Common Sense task, the agent is given the same scenario but with a different instruction that requires common-sense reasoning to identify the target. \cref{tab:env_nav_comparison} shows a comparison of the instructions for the two tasks under the same scenario.
\begin{table}[h]
    \centering\scriptsize
    \caption{Comparison of Base and Common Sense tasks in the Navigation environment. }\label{tab:env_nav_comparison}
    \begin{tabular}{p{0.48\textwidth}|p{0.48\textwidth}}
        \textbf{Base Instruction} & \textbf{Common Sense Instruction}  \\
        \midrule
        navigate to the Bread in the room and be as close as possible to it & I'm looking for freshly baked loaves that can be sliced for sandwiches or toast. Can you
        navigate to that object and stay close? \\\midrule
        navigate to the Pot in the room and be as close as possible to it & I need a large cooking vessel to prepare a hearty soup. Please navigate to that object and
        stay near it. \\\midrule
        navigate to the Toaster in the room and be as close as possible to it & I'm looking for a kitchen appliance that can brown and crisp slices of bread. Can you
        navigate to that object and stay close? \\
    \end{tabular}
    \end{table}
The training set contains 10,000 episodes, and the test set contains 128 episodes for each task. 
Below, we show an example of the first turn of the Navigation environment in both the training and evaluation stages.
\begin{academicprompt}{Navigation Example}
    \textbf{System Prompt} \\[1ex]
    You are a home robot and perform navigation tasks according to instructions. \\
\textcolor{promptbrown}{Actions you can take}: moveahead, moveback, moveright, moveleft, rotateright, rotateleft, lookup, lookdown.  \\
moveahead: Move forward by some distance \\
moveback: Move backward by some distance \\
moveright: Move rightward by some distance \\
moveleft: Move leftward by some distance \\
rotateright: Rotate to the right by 90 degrees
rotateleft: Rotate to the left by 90 degrees \\
lookup: Tilt the camera upward by 30 degrees \\
lookdown: Tilt the camera downward by 30 degrees\\
\textcolor{promptbrown}{Rewards:}\\
Format correct: +0.5 \\
Achieve the human instruction: +10.0 \\
The instruction will be provided with each observation. Look at the image carefully and navigate to complete the instruction.\\
\textcolor{promptbrown}{Hints:}\\
1. You can take multiple actions at a time, in most cases, if you find the target object is far away from you, you can call moveahead, moveleft and move right multiple times.\\
2. If you find yourself seems to be stuck, you can lookdown to see if there's any object above or below you, you can also rotate to see if there's any object behind you.\\
\textcolor{promptbrown}{Example:}\\
Round 1:\\
<image>\\
<think><observation>There is a garbage can in the upper left corner of the image, next to the kitchen sink. To move there, we can go forward-left, but since there's a kitchen counter directly ahead, we should go left first.</observation><reasoning>Following the strategy, I can go by first moving leftward.</reasoning><prediction>I will be infront of the garbage</prediction></think>
<answer>moveleft, moveleft</answer>
Round 2:\\
Env feedback: Last action is executed successfully.
<image>\\
<think><observation>From the scene, I see that by moving leftward, we are getting closer to the garbage can. Now, the garbage can is in front of me, slightly to the left. And there's a large area ahead of us.</observation><reasoning>Following the strategy, I can go by first moving forward then moving leftward.</reasoning><prediction>I will be closer to the garbage</prediction></think>\\
<answer>moveahead, moveahead,moveahead,moveleft</answer>\\
Round 3:\\
Env feedback: Last action is executed successfully.\\
<image>\\
<think><observation>From the image we can see the garbage can is very close to us, still to our front-left. Moving leftward might be blocked but i can see that there is still space in front of me to get closer to the garbage can.</observation><reasoning>Following the strategy, we can take about two steps forward then one step left to reach the garbage can.</reasoning><prediction>I will reach the garbage</prediction></think>
<answer>moveahead, moveahead,moveleft</answer>\\
Round 4:\\
Env feedback: Success \\
You can take up to 5 action(s) at a time, separated by ','.\\
You should first give your thought process with the your observation, reasoning, and prediction of next state, then your answer.\\
Both the observation and prediction should describe what you see or expect to see in the environment.\\
\textcolor{promptbrown}{Your response should be in the format of:}\\
<think><observation> ... </observation><reasoning> ... </reasoning><prediction> ... </prediction></think><answer> ... </answer>\\
\textcolor{promptbrown}{e.g.} <think><observation>I am at the entrance of a bedroom. There is a bed to the left, a desk with a lamp on the right, and a closet straight ahead. The target object, a book, appears to be on the desk.</observation><reasoning>I need to move toward the desk to reach the book. I'll turn right and move forward.</reasoning><prediction>I am now standing in front of the desk. The desk has a lamp, a computer, and several books on it. The target book is within reach on the right side of the desk.</prediction></think><answer> rotateright,moveahead,moveahead </answer>\\
\promptdashline
\begin{minipage}[c]{0.80\linewidth}
    \textbf{User Prompt} \\[1ex]
    <image>\\
    Human Instruction: I am looking for a luminous device on my desk to read books at night. Can you navigate to that object and stay close?\\
    Decide your next action(s).
\end{minipage}
\hspace{1em}%
\begin{minipage}[c]{0.15\linewidth}
    \centering
    \includegraphics[width=\linewidth]{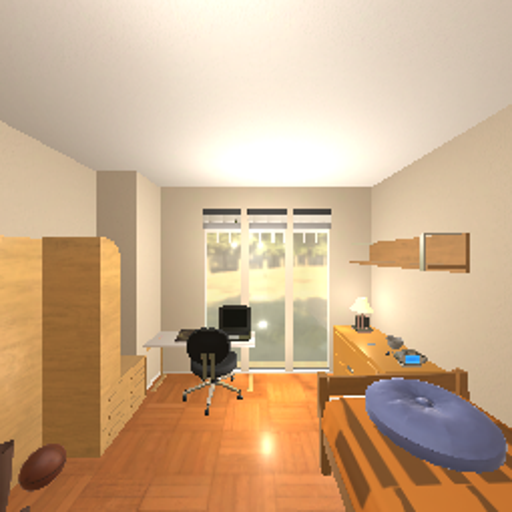}
\end{minipage}
\promptsolidline
\textcolor{promptblue}{
    \textbf{Model Output} \\[1ex]
    <think><observation>There is a laptop and a phone on the desk. The phone is slightly to the right and lower than the laptop. The luminous device is the phone.</observation><reasoning>Since the task is to find a luminous device on the desk, and the observation shows the phone, I should move towards the desk to get closer to the phone.</reasoning><prediction>I will be near the phone on the desk.</prediction></think><answer>moveleft, moveleft, moveleft, moveleft</answer>
}\\
\promptsolidline
\begin{minipage}[c]{0.80\linewidth}
    \textbf{User Prompt} \\[1ex]
    After your answer, the extracted valid action is [moveleft, moveleft, moveleft, moveleft].\\
    The environment feedback is: Last action is not executed successfully.\\
    reward: 0.5 \\
done: 0.0 \\
After that, the observation is:\\
<image>\\
Human Instruction: I am looking for a luminous device on my desk to read books at night. Can you navigate to that object and stay close?\\
    Decide your next action(s).\\
\end{minipage}
\hspace{1em}%
\begin{minipage}[c]{0.15\linewidth}
    \centering
    \includegraphics[width=\linewidth]{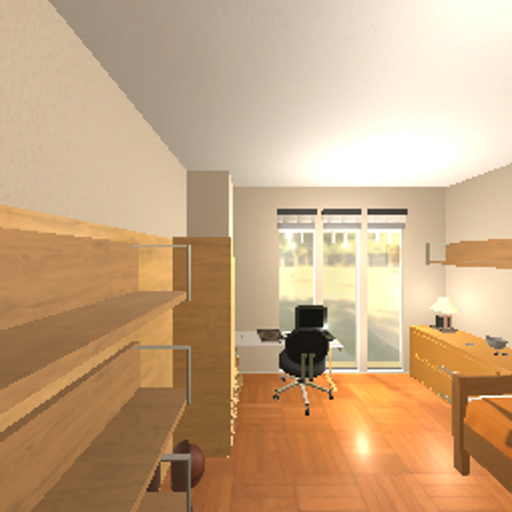}
\end{minipage}
\promptsolidline
...
\end{academicprompt}

\paragraph{Primitive Skill} 
We adopt the Primitive Skill environment originally introduced in \cite{hiranaka2023primitive} and implemented by \cite{vagen}. The goal of the task is to control a robot arm to accomplish several manipulation skills. The action space is
\begin{lstlisting}[language=Python]
    pick(x, y, z) # To grasp an object located at position(x,y,z) in the robot's workspace.
    place(x, y, z) # To place the object currently held by the robot's gripper at the target position (x,y,z).
    push(x1, y1, z1, x2, y2, z2) # To push an object from position (x1,y1,z1) to (x2,y2,z2).
\end{lstlisting}
The environment executes any valid action until the agent completes the task or the maximum number of turns (3) is reached. At each turn, the environment returns a format reward. At the end of the episode, the environment returns the success reward together with the step penalty and stage reward. The reward details are provided in \cref{tab:env_reward}. The four task variants are compared in \cref{tab:env_primitive_skill_comparison}.
During evaluation, we report the agent’s success rate on four tasks: Place, Stack, Draw, and Align.
\begin{table}[h]
    \centering\scriptsize
    \caption{Comparison of the four tasks in the Primitive Skill environment. }\label{tab:env_primitive_skill_comparison}
    \begin{tabular}{p{0.08\textwidth}|p{0.88\textwidth}}
        \textbf{Task} & \textbf{Example Instruction}  \\
        \midrule
        Place & Please place red cube at left target and green cube at right target.\\\midrule
        Stack & Please stack the red cube on top of the green cube, and then stack purple  cube on top of the red cube. \\\midrule
        Draw & Please put the apple in the drawer and close the drawer\\\midrule
        Align & Please align the cubes in the y-axis, which means the x-coordinates of both  cubes should be 0 (+-10mm)  
    \end{tabular}
    \end{table}
The training set is a mix of all four tasks, each containing 10,000 episodes. The test set contains 128 episodes for each task.
Below, we show an example of the first turn of the Primitive Skill environment in both the training and evaluation stages.
\begin{academicprompt}{Primitive Skill Example}
    \textbf{System Prompt} \\[1ex]
    You are an AI assistant controlling a Franka Emika robot arm. Your goal is to understand human instructions and translate them into a sequence of executable actions for the robot, based on visual input and the instruction.

{Action Space Guide}

\textcolor{promptbrown}{You can command the robot using the following actions:}

1. pick(x, y, z) \# To grasp an object located at position(x,y,z) in the robot's workspace.

2. place(x, y, z) \# To place the object currently held by the robot's gripper at the target position (x,y,z).

3. push(x1, y1, z1, x2, y2, z2) \# To push an object from position (x1,y1,z1) to (x2,y2,z2).

\textcolor{promptbrown}{Hints:} 

1. The coordinates (x, y, z) are in millimeters and are all integers.

2. Please ensure that the coordinates are within the workspace limits.

3. The position is the center of the object, when you place, please consider the volume of the object. It's always fine to set z much higher when placing an item.

4. We will provide the object positions to you, but you need to match them to the object in the image by yourself. You're facing toward the negative x-axis, and the negative y-axis is to your left, the positive y-axis is to your right, and the positive z-axis is up. 

\textcolor{promptbrown}{Examples:}

round1:

image1

Human Instruction: Put red cube on green cube and yellow cube on left target

Object positions:

[(62,-55,20),(75,33,20),(-44,100,20),(100,-43,0),(100,43,0)]

Reasoning: I can see from the picture that the red cube is on my left and green cube is on my right and near me. 

Since I'm looking toward the negative x axis, and negative y-axis is to my left, (62,-55,20) would be the position of the red cube, (75,33,20) would be the position of the green cube and (-44,100,20) is the position of the yellow cube. 

Also the (100,-43,0) would be the position of the left target, and (100,43,0) would be the position of the right target.

I need to pick up red cube first and place it on the green cube, when placing, I should set z much higher.

Answer: pick(62,-55,20)|place(75,33,50)

round2:

image2

Human Instruction: Put red cube on green cube and yellow cube on left target

Object positions:

[(75,33,50),(75,33,20),(-44,100,20),(100,-43,0),(100,43,0)]

Reasoning: Now the red cube is on the green cube, so I need to pick up the yellow cube and place it on the left target.

Answer: pick(-44,100,20)|place(100,-43,50)

You can take up to 2 action(s) at a time, separated by |.

You should first give your thought process with observation, reasoning, and prediction, and then your answer.

\textcolor{promptbrown}{Your response should be in the format of:}

<think><observation> ... </observation><reasoning> ... </reasoning><prediction> ... </prediction></think><answer> ... </answer>

\textcolor{promptbrown}{e.g.} <think><observation>I see a red cube at (100,100,40) and a green cube at (200,200,60).</observation><reasoning>I need to pick the red cube and place it on top of the green cube</reasoning><prediction>After executing this action, the red cube will be at position (200,200,100), stacked on top of the green cube at (200,200,60)</prediction></think><answer>pick(100,100,40)|place(200,200,100)</answer>
\promptdashline
\begin{minipage}[c]{0.80\linewidth}
    \textbf{User Prompt} \\[1ex]
    [Initial Observation]:

<image>

Human Instruction: Please put the apple in the drawer and close the drawer.

x\_workspace\_limit: (-500, -100)

y\_workspace\_limit: (-600, 200)

z\_workspace\_limit: (10, 650)

Object positions: 

[(-150, 126, 36), (-500, -537, 200)]

Other information:

\{'drawer\_open\_value': 263\}

Decide your next action(s).
\end{minipage}
\hspace{1em}%
\begin{minipage}[c]{0.15\linewidth}
    \centering
    \includegraphics[width=\linewidth]{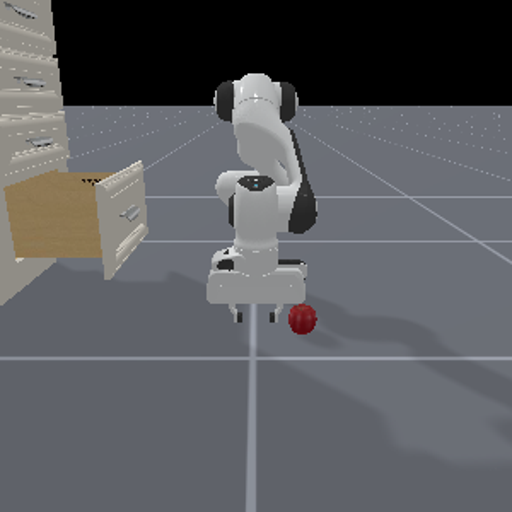}
\end{minipage}
\promptsolidline
\textcolor{promptblue}{
    \textbf{Model Output} \\[1ex]
    <think>I can see that the apple is located at position (-150, 126, 36), and the drawer is open at position (-500, -537, 200). The drawer is currently open, indicated by the drawer\_open\_value of 263. I need to pick the apple and place it in the drawer, then close the drawer.</observation>
}
\promptsolidline
\begin{minipage}[c]{0.80\linewidth}
    \textbf{User Prompt} \\[1ex]
    After your answer, the extracted valid action(s) is [].

After that, the observation is:

<image>

Human Instruction: Please put the apple in the drawer and close the drawer.

x\_workspace\_limit: (-500, -100)

y\_workspace\_limit: (-600, 200)

z\_workspace\_limit: (10, 650)

Object positions: 
[(-150, 126, 36), (-500, -537, 200)]

Other information:
\{'drawer\_open\_value': 263\}

Decide your next action(s).
\end{minipage}
\hspace{1em}%
\begin{minipage}[c]{0.15\linewidth}
    \centering
    \includegraphics[width=\linewidth]{images/env_example/primitive.png}
\end{minipage}
\promptsolidline
...
\end{academicprompt}

\paragraph{VIRL} 
We adopt the VIRL environment originally introduced in \cite{yang2024v,chu2025sft}. The goal of the task is to navigate to a destination on a map. We preserve the original task formulation and action space. 
\begin{lstlisting}[language=Python]
    forward() # Move forward one step along the road.
    turn_direction(x) # Adjust your direction towards x. x can be one of 8 compass directions: north, northeast, east, southeast, south, southwest, west, northwest.
    stop() # Indicate that you have reached the destination and the navigation is finished.
\end{lstlisting}
The original environment provides the ground-truth trajectory. For simplicity, we modify the environment to execute only the generated actions that are aligned with the ground-truth trajectory. At each turn, the environment returns a format reward, a key-point reward if any key points are reached, and an action reward if the action is correct. At the end of the episode, the environment returns a success reward to the agent. The reward details are provided in \cref{tab:env_reward}. During evaluation, we report the success rates on in-domain and out-of-domain tasks. In-domain tasks correspond to scenarios from the same city as the training data (NYC), while out-of-domain tasks come from Hong Kong, Buenos Aires, and New York. Below, we show an example of the first turn of the VIRL environment in both the training and evaluation stages.
\begin{academicprompt}{VIRL Example}
    \textbf{System Prompt} \\[1ex]
   \textcolor{promptbrown}{You are a navigation agent} following instructions on real-world streets using Google Street View images.

You observe a 2x2 grid of street view images from your current position. Top-left: front view, top-right: right view, bottom-left: back view, bottom-right: left view.

Your current facing direction will be indicated with each observation.

\textcolor{promptbrown}{Actions you can take}: forward(), turn\_direction(x), stop()

forward(): Move forward one step along the road.

turn\_direction(x): Adjust your direction towards x. x can be one of 8 compass directions: north, northeast, east, southeast, south, southwest, west, northwest.

stop(): Indicate that you have reached the destination and the navigation is finished.

The navigation instruction will be provided with each observation. Look at the street view images carefully and follow the instructions to navigate to the destination.

\textcolor{promptbrown}{Hints:}

1. Pay attention to landmarks and intersections mentioned in the instructions.

2. When you see an intersection, check your instructions to decide whether to turn or continue forward.

3. If the instruction requires turn, always turn no matter where you are facing. 

4. If the instruction requires move forward, you may need to do multiple move forward actions until you reach the required position. 

You can take up to 10 action(s) at a time, separated by `,'.

You should first give your thought process with your observation, reasoning, and prediction of next state, then your answer.

Both the observation and prediction should describe what you see or expect to see in the street view images.

\textcolor{promptbrown}{Your response should be in the format of:}

<think><observation> ... </observation><reasoning> ... </reasoning><prediction> ... </prediction></think><answer> ... </answer>

\textcolor{promptbrown}{e.g.}

Round 1:

[image]

Navigation Instruction:

[ ] 1. First, turn south to face south.

[ ] 2. Move forward until you reach next intersection.

[ ] 3. Turn right to face north.

[ ] 4. Move forward until you reach destination.

Movement history: (none)

<think><observation>I see a street view with buildings on both sides.</observation><reasoning>The instructions say: turn south, move forward, until the next intersection.</reasoning><prediction>I should first turn south and then move forward and see if there is intersection.</prediction></think>

<answer>turn\_direction(south), forward()</answer>

Round 2:

Env\_feedback: Executed 2 action(s). Reached checkpoint 1.

[image]

Navigation Instruction:

[x] 1. First, turn south to face south.

[ ] 2. Move forward until you reach next intersection.

[ ] 3. Turn right to face north.

[ ] 4. Move forward until you reach destination.

Movement history: turn\_direction(south), forward()

<think><observation>I am still on the road</observation><reasoning>From the instruction, i have not reached the next interaction, i need to keep forward</reasoning><prediction>>I need to forward until the next interaction</prediction></think>

<answer>forward()</answer>

\promptdashline
\begin{minipage}[c]{0.80\linewidth}
    \textbf{User Prompt} \\[1ex]
    You are currently facing northeast.

Navigation Instruction:

[ ] 1. First, turn slightly right to face northeast.

[ ] 2. Move forward until you reach next intersection.

[ ] 3. Turn right to face northwest.

[ ] 4. Move forward until you reach destination where The destination ReinoMed Pharmacy \& Stores is on your right behind.

Movement history: (none)

Decide your next action.
\end{minipage}
\hspace{1em}%
\begin{minipage}[c]{0.15\linewidth}
    \centering
    \includegraphics[width=\linewidth]{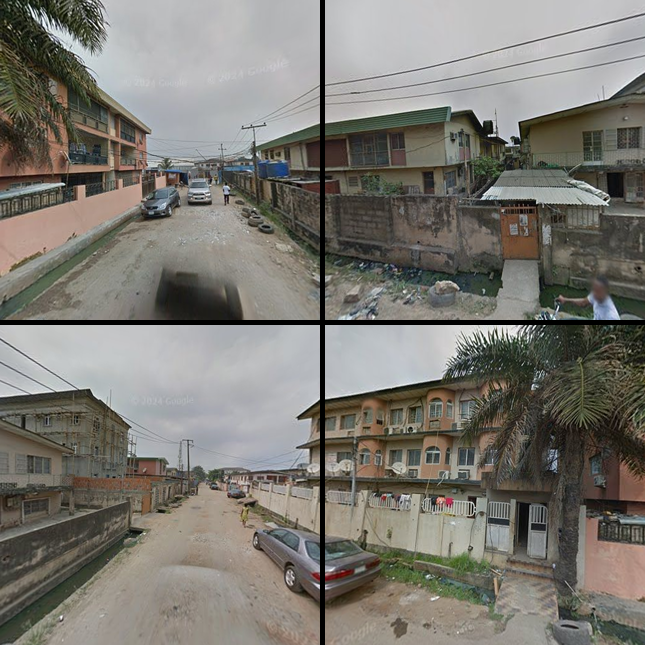}
\end{minipage}
\promptsolidline
\textcolor{promptblue}{
    \textbf{Model Output} \\[1ex]
    <think><observation>I am facing northeast</observation><reasoning>Instruction says to turn slightly right to face northeast.</reasoning><prediction>I will turn to face northeast.</prediction></think><answer>turn\_direction(northeast), forward()</answer>
}
\promptsolidline
\begin{minipage}[c]{0.80\linewidth}
    \textbf{User Prompt} \\[1ex]
    After your answer, the extracted valid action is [turn\_direction(northeast), forward()].

The environment feedback is: Executed 2 action(s). Reached checkpoint 1. Reached checkpoint 2

reward: 8.1

done: False

After that, the observation is:

<image>

You are currently facing northeast.

Navigation Instruction:

[x] 1. First, turn slightly right to face northeast.

[x] 2. Move forward until you reach next intersection.

[ ] 3. Turn right to face northwest.

[ ] 4. Move forward until you reach destination where The destination ReinoMed Pharmacy \& Stores is on your right behind.

Movement history: turn\_direction(northeast), forward()

Decide your next action.
\end{minipage}
\hspace{1em}%
\begin{minipage}[c]{0.15\linewidth}
    \centering
    \includegraphics[width=\linewidth]{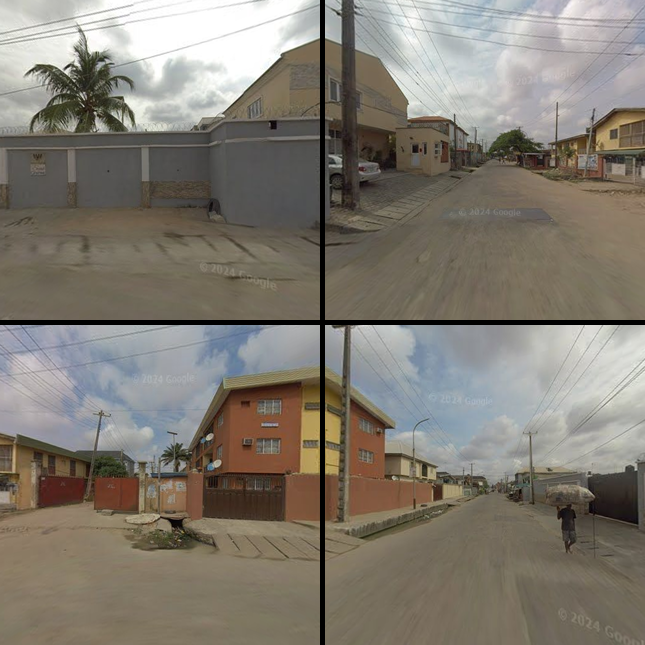}
\end{minipage}
\promptsolidline
...
\end{academicprompt}

\subsection{Training Details}

We follow the training procedure of \cite{yang2024v} for all tasks, with RL training supported by veRL~\cite{sheng2024hybridflow}. We train the models on NVIDIA A100, H100, and H200 GPUs.
The training hyperparameters are shown in \cref{tab:training_hyperparameters}. The environment-specific training hyperparameters are shown in \cref{tab:env_specific_training_hyperparameters}. 

\begin{table}[t]\centering
    \caption{Training hyperparameters.}
    \label{tab:training_hyperparameters}
    \begin{tabular}{lc}                                                                           
        \toprule                                                                        
        \textbf{Hyperparameter} & \textbf{Value} \\                                               
        \midrule               
        Base Model & Qwen2.5-VL-3B-Instruct \\                                                    
        Actor Learning Rate & $1 \times 10^{-6}$ \\                                               
        Critic Learning Rate & $1 \times 10^{-5}$ \\      
        Optimizer & AdamW \\                                        
        Train Batch Size & 128 \\                                                                 
        PPO Mini Batch Size & 32 \\
        KL Coefficient & 0.001 \\
        Entropy Coefficient & 0.001 \\
        Temperature & 0.7 \\
        Top-$p$ & 0.95 \\
        Advantage Estimator & HybAdV \\
        Discount Factor ($\gamma$) & 0.99 \\
        Mixing Coefficient ($\alpha$) & 0.5 \\
        \bottomrule
    \end{tabular}
\end{table}
\begin{table}[t]\centering\scriptsize
    \caption{Environment-specific training hyperparameters.}
    \label{tab:env_specific_training_hyperparameters}
    \begin{tabular}{lccccc}
        \toprule
        \textbf{Hyperparameter} & \textbf{Sokoban} & \textbf{FrozenLake} & \textbf{Navigation} &
    \textbf{Primitive Skill} & \textbf{VIRL} \\
        \midrule
        Max Response Length   & 200  & 200  & 200  & 200  & 150  \\
        Max Trajectory Length & 2400 & 2400 & 4000 & 4000 & 8192 \\
        Max Turns             & 3    & 3    & 4    & 3    & 7    \\
        Window Size           & 5    & 5    & 5    & 5    & 1    \\
        \#GPUs                & 2    & 2    & 4    & 4    & 8   \\
        Training Steps        & 300  & 300  & 100  & 100  & 100  \\
        \bottomrule
    \end{tabular}
\end{table}

\end{document}